\documentclass{article}

\usepackage{microtype}
\usepackage{graphicx}
\usepackage{subcaption}
\usepackage{booktabs} 
\usepackage{siunitx}    
\usepackage{hyperref}
\usepackage{amsmath}
\usepackage{amssymb}
\usepackage{mathtools}
\usepackage{amsthm}
\usepackage{geometry}

\usepackage{algorithm}
\usepackage{algorithmic}

\usepackage[numbers]{natbib}

\usepackage[capitalize,noabbrev]{cleveref}

\Crefname{example}{Example}{Examples}
\crefname{example}{example}{examples}

\theoremstyle{plain}
\newtheorem{theorem}{Theorem}[section]
\newtheorem{proposition}[theorem]{Proposition}
\newtheorem{lemma}[theorem]{Lemma}

\theoremstyle{definition}

\theoremstyle{remark}

\newtheorem{example}{Example}
\newcommand{\mnote}[1]{}
\newcommand{\YC}[1]{}
\newcommand{\AB}[1]{}
\renewcommand{\mnote}[1]{\textcolor{purple}{\textbf{[MU: #1]}}}
\renewcommand{\YC}[1]{\textcolor{cyan}{\textbf{[YC: #1]}}}
\renewcommand{\AB}[1]{\textcolor{blue}{\textbf{[AB: #1]}}}

\newcommand{\ournet}{{\text{SnareNet}}}
\newcommand{\adarel}{{\text{AdaRel}}}
\newcommand{\dc}{\text{DC3}}
\newcommand{\hardnet}{\text{HardNet}}

\newcommand{\pinet}{\text{$\Pi$net}}
\newcommand{\hproj}{\text{HProj}}
\newcommand{\optnet}{\text{OptNet}}

\newcommand{\relu}{\mathrm{ReLU}}

\newcommand{\st}{\text{ subject to }}
\newcommand{\ie}{\text{i.e., }}
\newcommand{\eg}{\text{e.g., }}

\newcommand{\ineq}{\text{ineq}}

\newcommand{\train}{\text{train}}
\newcommand{\test}{\text{test}}
\newcommand{\preimage}{\text{Preimage}}

\newcommand{\nom}{\text{nom}}

\DeclareMathOperator*{\argmin}{arg\,min}

\DeclareMathOperator*{\gmean}{gmean}
\DeclareMathOperator*{\mean}{mean}

\newcommand{\mbb}[1]{\mathbb{#1}}
\newcommand{\mcal}[1]{\mathcal{#1}}
\newcommand{\mbf}[1]{\mathbf{#1}}
\newcommand{\mbd}[1]{\boldsymbol{#1}}

\title{{\ournet}: Flexible Repair Layers for Neural Networks with Hard Constraints}

\author{Ya-Chi Chu\footnote{Department of Mathematics, Stanford University, CA, United States; email: ycchu97@stanford.edu} 
\and Alkiviades Boukas\footnote{Institute for Computational and Mathematical Engineering, Stanford University, CA, United States; email: aboukas@stanford.edu}
\and Madeleine Udell\footnote{Department of Management Science and Engineering, Stanford University, CA, United States; email: udell@stanford.edu} 
}

\begin{document}
\maketitle

\begin{abstract}
Neural networks are increasingly used as fast surrogate models across various domains, but unconstrained predictions can violate physical, operational, or safety requirements.
We propose SnareNet, a feasibility-controlled architecture to learn mappings whose outputs must satisfy input-dependent constraints.
SnareNet appends a differentiable repair layer that navigates in the constraint map's range space,
steering iterates toward feasibility and producing a repaired output that satisfies constraints to a user-specified tolerance.
We stabilize end-to-end training by adaptive relaxation,
a new training paradigm that snares the neural network at initialization and shrinks it into the feasible set,
enabling early exploration and strict feasibility later in training.
On optimization learning and trajectory planning benchmarks,
SnareNet consistently attains improved objective quality while satisfying constraints more reliably than prior work,
and it is the first to enforce \emph{non-convex} constraints at medium-to-high precision robustly across instances.
\end{abstract}

\section{Introduction}

Deep learning models have emerged as powerful tools across diverse applications, from computational biology \cite{jumper2021highly,angermueller2016deep} and drug discovery \cite{gomez2018automatic,wu2018moleculenet} to robotics \cite{levine2016end} and autonomous systems \cite{grigorescu2020survey}. Their success stems from an ability to learn complex patterns from data and make accurate predictions in high-dimensional spaces.
Neural networks (NNs) are increasingly being deployed as fast surrogate models that can replace or supplement traditional computational methods. For instance, in domains where the same type of problem must be solved frequently in real-time---such as power systems operation \cite{pan2021deepopf} and robotic control \cite{williams2017information}---or where individual solutions require expensive computation---such as weather forecasting \cite{lam2023learning,bonev2025fourcastnet} and fluid dynamics simulations \cite{kochkov2021machine}---NNs offer orders-of-magnitude speedups while maintaining competitive accuracy.
Standard NN architectures deliver unconstrained outputs.
While simple constraints like simplex constraints and coordinate-wise bounds can be readily enforced through architecture design,
it is challenging to enforce complex constraints, even for linear constraints.
Research on \emph{constrained neural networks} is growing in the past few years \cite{amos2017optnet,lu2021physics,iftakher2025physics}, driven by several compelling needs.

First, many critical applications require constrained outputs to ensure validity and safety.
In learning-based control policies, outputs must satisfy safety constraints and physical actuator limits to prevent catastrophic failures \cite{ames2016control}.
In applications such as optimal power flow in electrical grids or resource allocation in supply chains, outputs must satisfy operational constraints including power balance equations and capacity limits to deliver an executable solution \cite{pan2020deepopf}.
Violating these constraints can lead to solutions that are not only suboptimal but potentially dangerous or nonsensical in practice.

Secondly, constraints provide a mechanism to inject domain knowledge and inductive biases into NNs, which can potentially improve generalization, particularly when training data is limited.
In hierarchical time series forecasting \cite{cini2023graph,stratigakos2022end,rangapuram2021end}, enforcing aggregation constraints across levels of the hierarchy ensures coherent forecasts.
By encoding known structural properties of the problem into the network architecture, we can reduce the effective hypothesis space and regularize the learning process, leading to improved performance and robustness.

The most straightforward approach to impose constraints on NNs penalizes constraint violations in the loss function \cite{platt1987constrained,zhang1992lagrange,marquezneila2017imposinghardconstraintsdeep}, known as \emph{soft constraint methods}.
However, this approach provides no guarantees on constraint satisfaction at inference time. Typical gradient-based training with soft constraints produces constraint violations on the order of $10^{-2}$ or worse \cite{donti2021dc3}, leading to physically nonsensical or operationally infeasible solutions.
Perhaps more surprisingly, our numerical experiments in \Cref{app:soft} reveal that soft constraint training can be counterproductive even as a warm-start strategy.
These observations underscore the necessity for hard constraint enforcement mechanisms that provide both feasibility and efficient training.

In this paper, we introduce {\ournet} (\Cref{fig:snarenet-adarel}), a novel framework to enforce nonlinear, input-dependent constraints on outputs of NNs while preserving their approximation capabilities and enabling end-to-end training. Our key contributions are:
\begin{itemize}


    \item \textbf{Adaptive relaxation training paradigm:} We introduce a novel training strategy that progressively tightens constraint satisfaction during training (\Cref{fig:adarel}), improving optimality achieved by the final model while maintaining strict feasibility at the end of training.

    \item \textbf{Feasibility control even for non-convex constraints:} {\ournet} enables users to explicitly control the desired level of constraint satisfaction, allowing for empirical trade-offs between strict feasibility and solution quality.
        {\ournet} is the first repair layer to consistently achieve tight feasibility even on non-convex constraints.

    \item \textbf{Broad applicability and robustness:} We demonstrate {\ournet}'s effectiveness on optimization learning tasks and neural control policies.
    {\ournet} consistently produces feasible solutions with better objective values compared to state-of-the-art baselines.
\end{itemize}

\paragraph{Organization.}
\Cref{sec:constrained_nn} introduces the problem of constrained NN and review the literature.
\Cref{sec:motivation} illustrates a novel perspective on the existing closed-form repair layer for linear constraints and motivates {\ournet}.
\Cref{sec:snarenet} presents our {\ournet} for nonlinear constraints and the adaptive relaxation training paradigm.
\Cref{sec:exp} demonstrates our experiments on {\ournet}.
\Cref{sec:conclusion} concludes with a discussion of future directions.
We interleave related literature review into \Cref{sec:constrained_nn} and \Cref{sec:motivation} since literature were integral to the development of {\ournet}.

\begin{figure}
    \centering
    \begin{subfigure}[b]{0.58\linewidth}
        \centering
        \includegraphics[width=\linewidth]{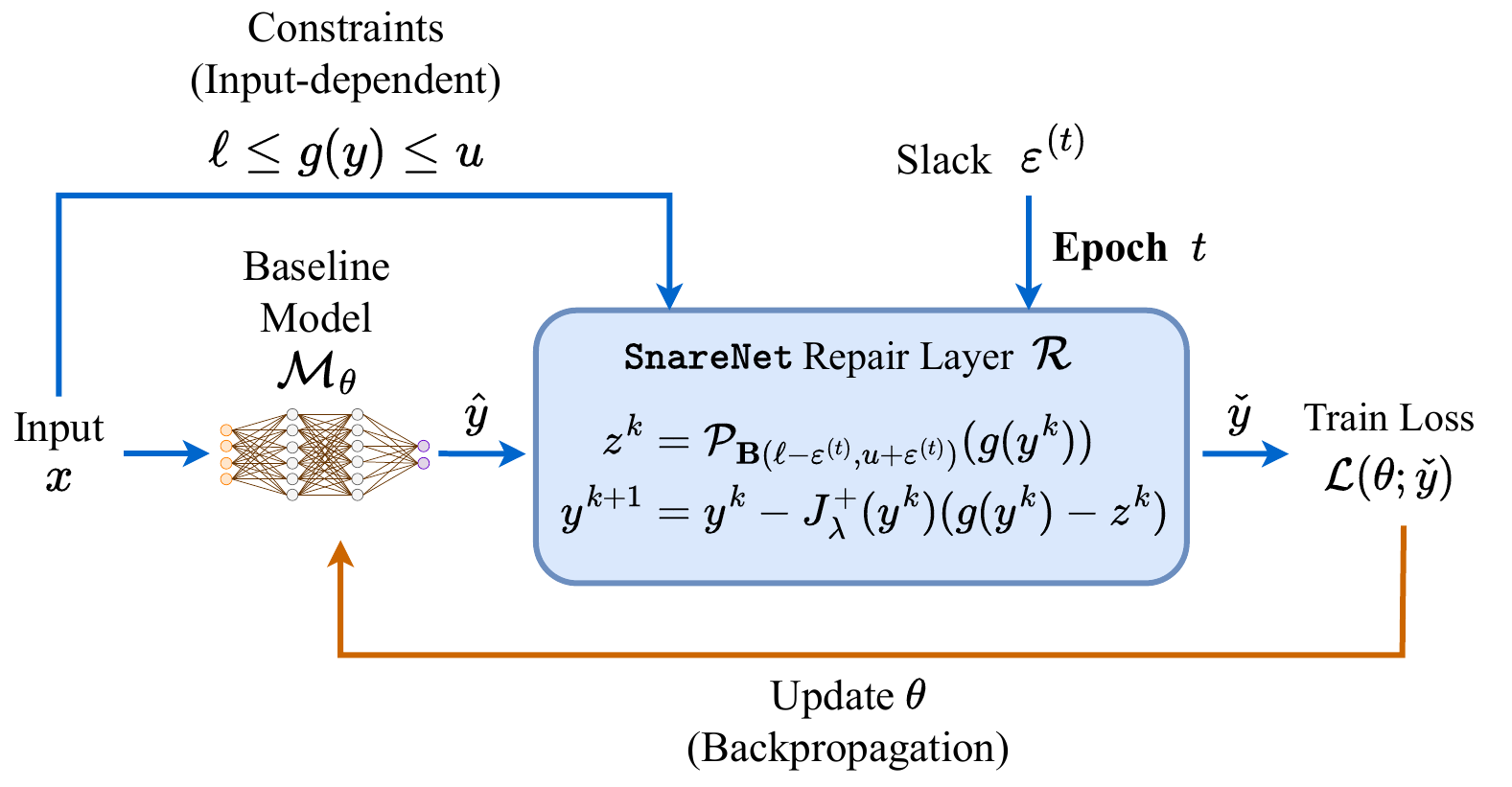}
        \caption{{\ournet} is equipped with an iterative repair layer.}
        \label{fig:snarenet}
    \end{subfigure}
    \hfill
    \begin{subfigure}[b]{0.4\linewidth}
        \centering
        \includegraphics[width=\linewidth]{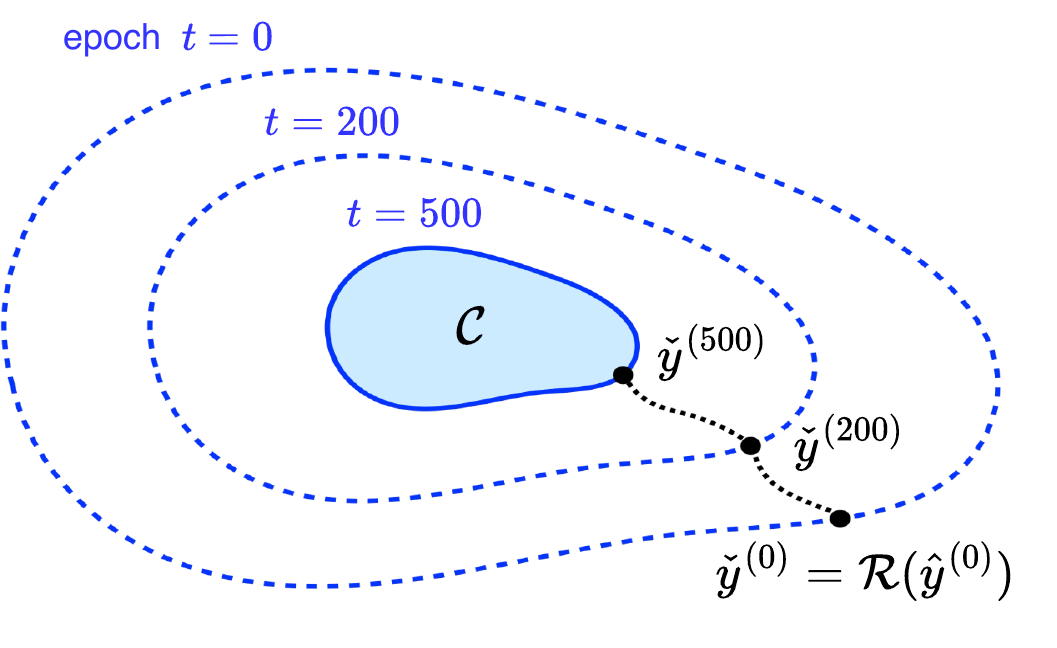}
        \caption{
            Adaptive relaxation training paradigm.
        }
        \label{fig:adarel}
    \end{subfigure}
    \caption{{\ournet}'s architecture design and training paradigm.}
    \label{fig:snarenet-adarel}
\end{figure}

\section{Constrained Neural Networks} \label{sec:constrained_nn}

The fundamental task in deep learning is to learn complex input-output mappings $\Phi: \mcal{X} \rightarrow \mcal{Y}$ from data.
Given input $x \in \mcal{X}$, the deep learning model $\mcal{M}_{\theta}: \mcal{X} \rightarrow \mcal{Y}$ parametrized by weights $\theta$ is expected to produce an output $\hat{y} = \mcal{M}_{\theta}(x)$ that approximates $y = \Phi(x)$.
These deep learning models are typically trained by minimizing an empirical risk over a finite training dataset $\mcal{X}_{\train} \subset \mcal{X}$ defined by $\mathcal{L}(\theta) = \frac{1}{|\mathcal{X}_{\train}|} \sum_{x \in \mathcal{X}_{\train}} \ell(\theta; x)$, where $\ell(\theta; x)$ denotes a suitable loss function measuring the discrepancy between predictions and targets for input $x$ under model parameter $\theta$.
For simplicity, we consider the setting where $\mcal{X} \subset \mbb{R}^d$ and $\mcal{Y} \subset \mbb{R}^n$ admit vector representations.

In this paper, we consider the task to impose input-dependent (\ie $x$-dependent) constraints on the output $y$, which take the form:
\begin{equation} \label{eq:constraints}
\ell_x \leq g_x(y) \leq u_x, \quad (\text{abbrev. }~ \ell \leq g(y) \leq u)
\end{equation}
where $g : \mbb{R}^n \rightarrow \mbb{R}^m$ is differentiable, $\ell, u \in (\mbb{R} \cup \{\pm \infty \})^m$ are lower and upper bounds respectively, and all $g, \ell, u$ are parametrized by input $x$.
We suppress the dependence of $g, \ell, u$ on $x$ throughout the paper for notational simplicity.
The formulation in \eqref{eq:constraints} encompasses equality constraints when $\ell_i = u_i$ for index $i$, as well as one-sided constraints when $\ell_i = -\infty$ or $u_i = \infty$.
Throughout the paper, we assume the feasible set $\mcal{C} := \{ y \in \mcal{Y} \mid \ell \leq g(y) \leq u \}$ is non-empty for any $x \in \mcal{X}$.

\subsection{Soft-Constraint Methods}
To encourage the model $\mcal{M}_{\theta}$ to produce feasible outputs, early approaches penalize constraint violations of the output\footnote{We denote the output by $\hat{y}$ for simplicity when the dependence on $\theta$ and $x$ is clear from the context.} $\hat{y} = \hat{y}(\theta, x) := \mcal{M}_{\theta}(x)$ in the loss function:
\begin{equation} \label{eq:penalty-loss} \begin{aligned}
\ell_{\text{soft}}(\theta; x) = \ell(\theta; x) + \mu_u \|\relu(g(\hat{y}) - u)\|^2 + \mu_{\ell} \|\relu(\ell - g(\hat{y}))\|^2,
\end{aligned}
\end{equation}
where $\mu_u > 0$ and $\mu_{\ell} > 0$ are penalty weights that determine the strength of the penalty, and $\relu(\cdot) := \max(0, \cdot)$ is applied elementwise.
The loss in \eqref{eq:penalty-loss} is also called a \emph{soft loss} and allows the use of standard unconstrained optimization techniques during training.
This soft constraint methodology is commonly used for data-parameterized constrained optimization \cite{van2025optimization} and for solving partial differential equations \cite{raissi2019physics, dener2020trainingneuralnetworksphysical}.
While penalty methods are straightforward to implement and broadly applicable,
models trained with soft loss generally violate constraints on unseen problem instances.
Strict constraint satisfaction requires setting $\mu_u$ and $\mu_{\ell}$ to infinity (or a very large number), which however makes the soft loss infinitely ill-conditioned and hard to solve \cite{rathore2024challenges}.

\subsection{Hard-Constraint Methods}
To impose the constraints \eqref{eq:constraints} on the output, a more sophisticated approach appends a \emph{repair module\footnote{
    We borrow the word ``repair'' from the term \emph{repair layer} in \citet{van2025optimization},
    but we call $\mcal{R}$ a \emph{layer} only if it is trainable.
    The same concept has also been referred to as
    a \emph{projection layer} \cite{min2024hardnet,grontas2025pinet,liang2024homeomorphic} or a \emph{feasibility layer} \cite{li2023learning,nguyen2025fsnet} in the literature when the repair module is trainable.}
} $\mathcal{R}: \mathbb{R}^n \rightarrow \mathbb{R}^n$ that maps any potentially infeasible output $\hat{y}$ to a feasible one $\check{y}$.
We use $\hat{y} = \mcal{M}_{\theta}(x)$, pronounced ``y-hat'', to denote the approximate solution from the baseline model, which is potentially infeasible;
and we use $\check{y} = \mcal{R}(\hat{y})$, pronounced ``y-check'', to denote the solution after applying the repair module, which is enforced to be feasible.
Two paradigms have appeared to couple repair modules to a baseline NN,
post-processing and end-to-end training.

\paragraph{Post-processing for Feasibility.} 
The post-processing approach applies the repair module only during inference, but do not allow differentiable training for model parameters $\theta$.
The simplest post-processing approach is projection onto the feasible set.
Given any test instance $x_{\test}$ and a trained model $\mcal{M}_{\theta}$, the repair module projects the baseline output onto the feasible set, \ie
\begin{equation} \label{eq:opt-proj}
\mathcal{R}(\hat{y}(\theta, x_{\test})) := \argmin_{y \in \mcal{C}} \|y - \mcal{M}_{\theta}(x_{\test}) \|^2.
\end{equation}
When $\mcal{C}$ is convex, problem \eqref{eq:opt-proj} is a convex program that typically can be reliably solved by existing optimization solvers.
When $\mcal{C}$ is non-convex, optimization solvers still reliably find feasible (but not always optimal) solutions to \eqref{eq:opt-proj}.
Alternatively, \citet{liang2024homeomorphic} propose to use an auxiliary deep learning model that learns a homeomorphic mapping from the feasible set to the unit ball.
However, their method needs binary search to project an infeasible solution to the feasible set, which precludes end-to-end training of the combined system.
Post-processing can lead to critical problems as demonstrated in \citet[Appendix C.4]{grontas2025pinet}: models trained without active constraints may diverge on unbounded objectives.
Even when they converge, solution quality may be poor after projection since the baseline model fails to anticipate the repair operation.

\paragraph{Trainable Layers for Feasibility.}
On the other hand, the end-to-end approach activates the repair layer during both training and inference, allowing gradients of model parameters $\theta$ to flow through the repair operation. The baseline model $\mcal{M}_{\theta}$ can therefore adapt to the repair operation.
To enable backpropagation, the repair layer must use only operations that are differentiable almost everywhere.
{\optnet} \cite{amos2017optnet} proposes differentiable convex optimization problem as a layer, for which the gradients are calculated from the KKT conditions in backward pass.
Under convex constraints, the convex projection problem \eqref{eq:opt-proj} can be appended as a repair layer, albeit at the cost of expensive training.
State-of-the-art lightweight trainable repair layers typically employ a fixed number of iterative updates derived from classical optimization algorithms, such as {\dc} \cite{donti2021dc3} and {\pinet} \cite{grontas2025pinet}.
{\dc} predicts an initial solution using a soft penalty formulation and moves it toward the feasible region via equality completion and gradient descent on the distance to feasibility.
{\pinet} decomposes an affine feasible set as the intersection of two convex sets in lifted dimension and applies Douglas-Rachford algorithm, in which the computational efficiency is improved by restricting the framework to linear and box constraints.
While iterative repair layers improve constraint satisfaction, backpropagation through these iterations requires intensive GPU memory.
In contrast, {\hardnet} \cite{min2024hardnet} uses a closed-form repair layer that guarantees exact feasibility up to machine precision and enables efficient backpropagation, but {\hardnet} is limited to linear constraints and often produces suboptimal solutions.

\paragraph{Key Challenges.}
A trainable repair layer must find a feasible point using only differentiable operations.
Differentiability can be achieved by unrolling iterative updates,
but repair layers that use too many iterations to reach feasibility generate large unrolled computational graphs,
resulting in slow training.
Convex optimization layers \cite{amos2017optnet} avoid unrolling, but cannot handle non-convex constraints.
Designing an efficient repair layer for non-convex constraints remains an open challenge.

\section{Novel Preimage Perspective From Linear Constraints} \label{sec:motivation}

We develop a new interpretation of the closed-form repair layer for full row rank linear constraints proposed by {\hardnet} \cite{min2024hardnet}.
Our preimage perspective interprets the inequality feasibility problem as a more tractable equation solving problem.
Unlike the original interpretation from {\hardnet}, our preimage perspective suggests a natural generalization to nonlinear and nonconvex constraints.


\subsection{Preimage Perspective for Closed-Form Linear Repair Layer}
{\hardnet} \cite{min2024hardnet} shows that iterative
repair layers are unnecessary by introducing a closed-form repair layer that strictly enforces all linear constraints
$\ell \leq A y \leq u$ when $A \in \mbb{R}^{m \times n}$ has full row rank:
\begin{align}
    \mcal{R}(\hat{y}) &:= \hat{y} + A^{+} \mbd{\delta}(A \hat{y}; \ell, u), ~\mbox{ where}~
    \mbd{\delta}(z; \ell, u) := \text{ReLU}(\ell - z) - \text{ReLU}(z - u) \label{eq:hardnet-layer}
\end{align}
and $A^{+}$ is the pseudoinverse of $A$.
The \emph{correction vector} $\mbd{\delta}(z; \ell, u)$ corrects $z$ to the box $\mbf{B}(\ell, u) := \{ z \in \mbb{R}^m \mid \ell_i \leq z_i \leq u_i,~ \forall i \in [m] \}$ by elementwise projection since
\begin{equation} \label{eq:box-proj}
    \mcal{P}_{\mbf{B}(\ell, u)} (z)
    := \begin{cases}
    z_i, &\text{ if } \ell_i \leq z_i \leq u_i; \\
    \ell_i, &\text{ if } z_i \leq \ell_i; \\
    u_i, &\text{ if } z_i \geq u_i;
    \end{cases}
    = z + \mbd{\delta}(z; \ell, u).
\end{equation}
Observe that the preimage of $\mbf{B}(\ell, u)$ under the map $A$ is exactly the feasible set $\mcal{C}$:
\begin{align*}
    \mcal{C} = \{y \in \mbb{R}^n \mid \ell \leq A y \leq u \}
    = \{y \in \mbb{R}^n \mid A y \in \mbf{B}(\ell, u) \}
    = \preimage \text{ of } \mbf{B}(\ell, u) \text{ under } A.
\end{align*}
We interpret {\hardnet} through the preimage: {\hardnet} projects the image vector $A \hat{y}$ onto the box $\mbf{B}(\ell, u)$ and chooses a particular feasible $\check{y} = \mcal{R}(\hat{y})$ that lies in the preimage of $\mcal{P}_{\mbf{B}(\ell, u)} (A \hat{y})$.
That is, we can \emph{reformulate the inequality feasibility problem as a tractable equation solving problem}: $A \check{y} = \mcal{P}_{\mbf{B}(\ell, u)} (A \hat{y})$ using box projector; and moreover, the bix projector admits a closed-form expression \eqref{eq:box-proj}.
When $A$ has full row rank, {\hardnet} solves for $\check{y}$ by updating $\hat{y}$ with the vector $A^+\mbd{\delta}(A \hat{y}; \ell, u)$ that is pulled back from the correction $\mbd{\delta}(A \hat{y}; \ell, u)$ in the constraint space by the pseudoinverse $A^+$ (\Cref{fig:procedure-linear}).
Linearity of $A$ ensures $A^+\mbd{\delta}(A \hat{y}; \ell, u)$ serves as a feasibility correction:
\begin{equation*}
    A \mcal{R}(\hat{y})
    = A ( \hat{y}+A^+\mbd{\delta}(A \hat{y}; \ell, u) )
    = A\hat{y} + \mbd{\delta}(A \hat{y}; \ell, u)
    = \mcal{P}_{\mbf{B}(\ell, u)} (A \hat{y}).
\end{equation*}

\begin{figure}
    \centering
    \begin{subfigure}[b]{0.49\linewidth}
        \centering
        \includegraphics[width=\linewidth]{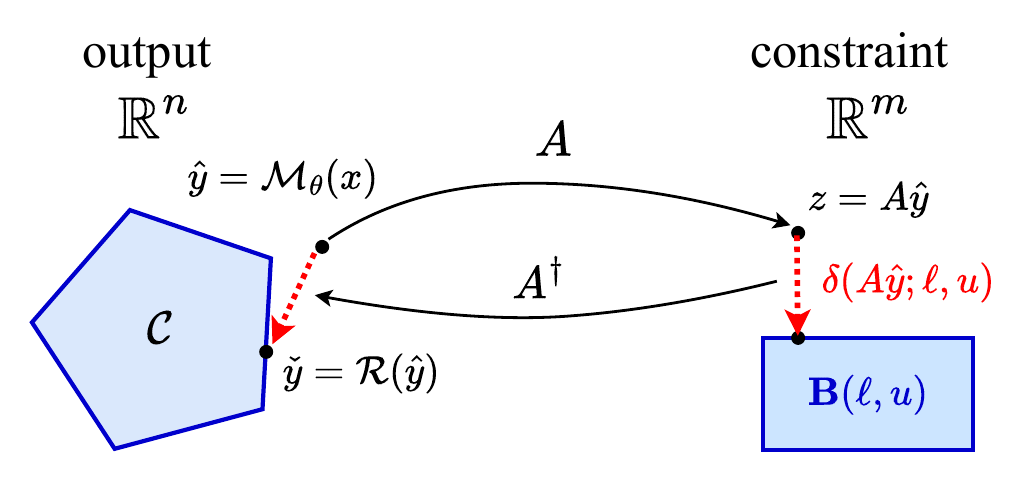}
        \caption{Linear $g(y) = Ay$ with full row rank $A$}
        \label{fig:procedure-linear}
    \end{subfigure}
    \hfill
    \begin{subfigure}[b]{0.49\linewidth}
        \centering
        \includegraphics[width=\linewidth]{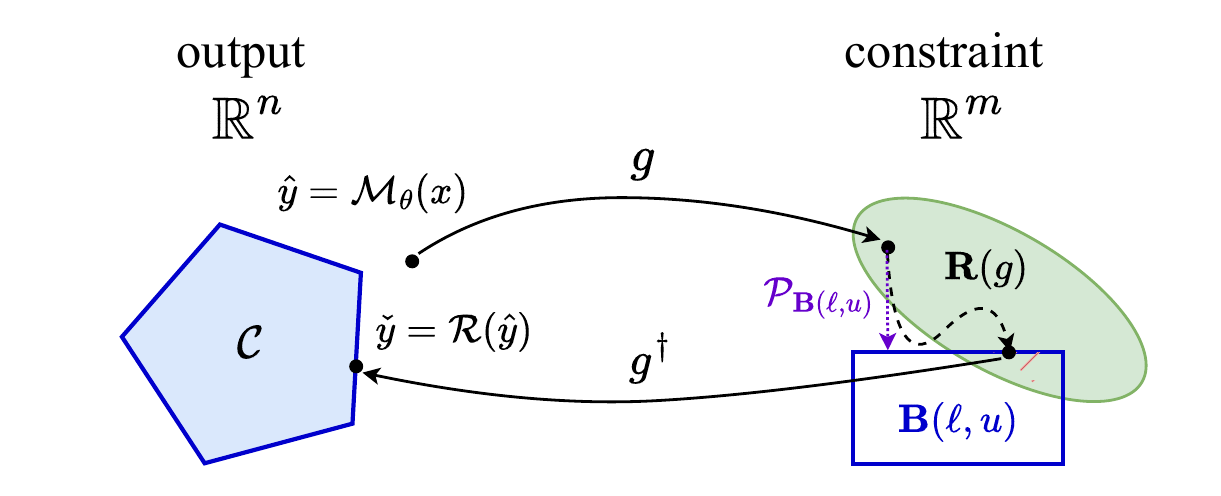}
        \caption{Nonlinear $g(y)$}
        \label{fig:procedure-nonlinear}
    \end{subfigure}
    \caption{
    Preimage perspective to reformulate inequality feasibility problems to tractable equation solving problems.
    {\hardnet} corrects an infeasible prediction $\hat{y}$ by a correction pulled back from the constraint space to the output space (\Cref{fig:procedure-nonlinear}).
    {\ournet} finds a path to approach the intersection $\mathbf{R}(g) \cap \mathbf{B}(\ell, u)$ (\Cref{fig:procedure-nonlinear}).
    }
    \label{fig:procedure}
\end{figure}

\subsection{The Challenge of Non-Linear Constraints}

The above reformulation that uses the box projector to convert an inequality feasibility problem to equations solve does not directly extend to nonlinear constraints. \Cref{ex:disk} illustrates the challenge:

\begin{example}\label{ex:disk}
    Constrain $y \in \mbb{R}^2$ to the intersection of two disks of radius $3/2$, centered at $(\pm 1, 0)$:
    \begin{align*}
        g(y) = \begin{bmatrix}
            g_1(y) \\ g_2(y)
        \end{bmatrix}
        :=
        \begin{bmatrix}
            (y_1 + 1)^2 + y_2^2 \\ (y_1 - 1)^2 + y_2^2
        \end{bmatrix}
        \leq u :=
        \begin{bmatrix}
            9/4 \\ 9/4
        \end{bmatrix}.
    \end{align*}
    Now, the infeasible prediction $\hat{y} = (-1, 0)$ violates the second constraints but satisfies the first one as $g_1(\hat{y}) = 0$ and $g_2(\hat{y}) = 4$.
    The box projection of $g(\hat{y})$ is $\mathcal{P}_{\mathbf{B}(-\infty, u)}(g(\hat{y})) = (0, \frac{9}{4})$. However, there is \emph{no} point $y \in \mathbb{R}^2$ such that $g(y) = (0, \frac{9}{4})$ since the system of equations
    \begin{equation*}
        \begin{cases}
        g_1(y_1, y_2) = (y_1 + 1)^2 + y_2^2 = 0 \\
        g_2(y_1, y_2) = (y_1 - 1)^2 + y_2^2 = \frac{9}{4}
        \end{cases}
    \end{equation*}
    has no solution
    and the preimage of $(0, \frac 9 4)$ under $g$ is empty.
\end{example}
\Cref{ex:disk} reveals the fundamental difficulty: the \emph{joint numerical range} $\mbf{R}(g) := \{g(y) \in \mbb{R}^m \mid y \in \mbb{R}^n \}$ of non-linear $g$ is typically a proper subset of $\mbb{R}^m$.
Hence, the projection of $g(\hat{y})$ onto $\mathbf{B}(\ell, u)$ may have empty preimage under $g$.
To guarantee feasibility, we must move towards the intersection $\mathbf{R}(g) \cap \mathbf{B}(\ell, u)$, which characterizes points that are both (i) in the range of $g$ (ensuring a preimage exists) and (ii) within the feasible box (ensuring constraint satisfaction). See \Cref{fig:procedure-nonlinear}.

\section{{\ournet}: Design and Analysis} \label{sec:snarenet}
This section introduces  {\ournet} (\Cref{fig:snarenet-adarel}), which efficiently finds a feasible point satisfying \eqref{eq:constraints}.

\subsection{Repair Layer: Regularized Adaptive Newton Update}

Given $z \in \mathbf{R}(g) \cap \mathbf{B}(\ell, u)$, Newton's method can find a feasible $\check{y} \in \mcal{C}$ by solving the non-linear system $g(y) = z$.
{\ournet} uses $z^k = \mcal{P}_{\mbf{B}(\ell, u)} ( g(y^k) )$ as the initial RHS at each Newton iteration:
\begin{align}
    y^{k+1} &= \argmin_y \| J_{g}(y^k) (y - y^k) + g(y^k) - z^k \|^2
    = y^k - J_{g}(y^k)^{+} (g(y^k) - z^k), \label{eq:newton-update}
\end{align}
where $J_{g}(y) \in \mbb{R}^{m \times n}$ is the Jacobian of $g$.
When $g = A$ has full row rank, $z^k$ must admit a preimage, Newton's method converges in one iteration, and the update \eqref{eq:newton-update} reduces to \eqref{eq:hardnet-layer} in {\hardnet}.
The linearization in Newton's method \eqref{eq:newton-update} approximates $g$ only locally and thus requires safeguards to ensure convergence.
{\ournet} uses Levenberg–Marquardt (LM) regularization,
replacing \eqref{eq:newton-update} by
\begin{equation} \label{eq:LM-update}
    y^{k+1} = \argmin_y \| J_{g}(y^k) (y - y^k) + g(y^k) - z^k \|^2 + \lambda \| y - y^k \|^2
    = y^k - J^{\dagger}_{\lambda}(y^k) (g(y^k) - z^k),
\end{equation}
where $J^{\dagger}_{\lambda}(y^k) := [J_{g}(y^k)^{\top} J_{g}(y^k) + \lambda I ]^{-1} J_{g}(y^k)^{\top}$.
The regularization term $\lambda \| y - y^k \|^2$ with $\lambda > 0$ ensures the new point is close to $y^k$, at which $g$ is linearized.
{\ournet} terminates when the update magnitude or constraint violation falls below a prescribed tolerance or the iteration limit is reached.

\subsection{Training Paradigm: Adaptive Relaxation}

Strict constraints enforcement during early-stage training can harm performance.
\citet{min2024hardnet} observes that immediate enforcement prevents models from reaching better-optimized final states, and solves the problem by disabling the repair layer and using soft loss \eqref{eq:penalty-loss} in the first few epochs, also called ``soft-epochs''.
\citet{nguyen2025fsnet} show that randomly initialized networks produce large violations that force the repair layer to work overtime and worsen the final solution.

{\ournet} instead uses a new \emph{adaptive relaxation} ({\adarel}) training paradigm that
snares the neural network at initialization and shrinks it into the feasible set
throughout the training process (\Cref{fig:adarel}).
In the beginning of training, we repair the approximate solution $\hat{y}$ towards a relaxed feasible set 
\begin{equation*}
    \mcal{C}_{\varepsilon^{(t)}} := \{y \in \mbb{R}^n \mid \ell - \varepsilon^{(t)} \leq g(y) \leq u + \varepsilon^{(t)} \},
\end{equation*}
parametrized by a slack $\varepsilon^{(t)} \geq 0$ at epoch $t$.
Our experiments use a slack $\varepsilon^{(t)}$ linearly decaying to zero over a preset decay horizon, allowing the model to explore a broader solution space initially while gradually tightening the constraints.
We set $\varepsilon^{(t)} = 0$ for the last few training epochs to ensure exact feasibility.
At epoch $t$, we implement {\adarel} by replacing $z^k$ in \eqref{eq:LM-update} with
\begin{equation} \label{eq:z-adarelx}
    z^k := \mcal{P}_{\mbf{B}(\ell - \varepsilon^{(t)}, u + \varepsilon^{(t)})} ( g(y^k) ).
\end{equation}





\subsection{Analysis through the Lens of Preconditioning} \label{sec:conv}

{\ournet}'s update \eqref{eq:LM-update} is equivalent to \emph{preconditioned gradient descent} (PGD) on constraint violation
\begin{equation*} 
    F(y) = \frac{1}{2} \| g(y) - \mathcal{P}_{\mathbf{B}(\ell, u)}(g(y)) \|^2
\end{equation*}
with preconditioner $P_\lambda(y) := (J_g(y)^\top J_g(y) + \lambda I)^{-1}$ for $\lambda > 0$.
\Cref{prop:F-property} shows that $F(y)$ is differentiable if $g$ is and $\nabla F(y) = J_g(y)^\top (g(y) - \mathcal{P}_{\mathbf{B}(\ell, u)}(g(y)))$.
When $\mcal{C}$ has a convex representation,
\begin{equation*} 
    \mathcal{C} := \{ y \in \mathbb{R}^n \mid h(y) \leq 0,\, \ell_A \leq Ay \leq u_A \}
\end{equation*}
with convex $h : \mathbb{R}^n \rightarrow \mathbb{R}^{m_h}$,
matrix $A \in \mathbb{R}^{m_A \times n}$,
and bounds $\ell_A, u_A \in \mathbb{R}^{m_A}$,
$F$ is convex.
Moreover, if $g(y) = Ay$ and $A$ has full row rank, $F$ satisfies the Polyak--\L{}ojasiewicz (PL) condition in the preconditioned norm $\| \cdot \|_{P_\lambda(y)}$ (see \Cref{lem:F-PL}).
The convergence guarantees of {\ournet} are grounded in PGD update and these properties of $F$ such as convexity and Polyak-Lojasiewicz (PL) condition.
\Cref{thm:linear-main} demonstrates our preconditioner $P_\lambda(y)$ enables a fast linear convergence rate for full row rank linear constraints.
\Cref{thm:nonlinear-main} justifies asymptotic feasibility under convex constraints.
{\ournet} minimizes constraint violation by PGD for non-convex constraints.

\begin{theorem}[Linear Convergence for Full Row Rank Linear Constraints]
\label{thm:linear-main}
Suppose $\mathcal{C} := \{ y \in \mathbb{R}^n \mid \ell
\leq A y \leq u \}$ is non-empty and $A$ has full row rank with minimum singular value $\sigma_{\min} > 0$. Then {\ournet} with $\lambda > 0$ converges linearly to a feasible point at rate 
\begin{equation*}
    F (y^k) \leq \left( 1 - \frac{\sigma^2_{\min}}{\sigma_{\min}^2 + \lambda}
    \right)^k F (y^0) .
\end{equation*}
\end{theorem}

\Cref{thm:linear-main} suggests that a smaller $\lambda$ leads to faster convergence and $\lambda = 0$ matches one step convergence as in {\hardnet}.
However, our empirical results (\Cref{fig:noncvx-lambda}) show that a small $\lambda$ can lead to feasible solutions with worse optimality.

\begin{theorem}[Convergence for Convex Constraints]
\label{thm:nonlinear-main}
Suppose $\mathcal{C} := \{ y \in \mathbb{R}^n \mid h(y) \leq 0,\, \ell_A \leq Ay \leq u_A \}$ is non-empty,
where $h : \mathbb{R}^n \rightarrow \mathbb{R}^{m_h}$ is a twice differentiable convex function with $L$-Lipschitz continuous Jacobian,
$A \in \mathbb{R}^{m_A \times n}$,
and $\ell_A, u_A \in \mathbb{R}^{m_A}$.
Let $y^0 \in \mathbb{R}^n$ be the initial point, $\mathcal{D} := \{ y \in \mathbb{R}^n \mid F (y) \leq F (y^0) \}$ be the sublevel set, and
$D := \sup_{y \in \mathcal{D}} \inf_{y^{\ast} \in \mathcal{C}} \| y
- y^{\ast} \|$.
Suppose $\sigma^2_{\max} (J_h (y)) + \sigma^2_{\max} (A) \leq \sigma^2$ for all
$y \in \mathcal{D}$ and $\lambda > 0$ satisfies $\lambda^2 - \sqrt{2 F (y^0)} L
\lambda + \tfrac{L^2 F (y^0)^2}{16} \geq 0$. Then {\ournet} converges to
a feasible point at rate 
\begin{equation*}
F (y^k) \leq \frac{F (y^0)}{1 + \frac{F (y^0)}{2 D^2 (\sigma^2 + \lambda)}k}.
\end{equation*}
\end{theorem}
Convex constraints requires sufficiently large $\lambda$ for convergence while too large $\lambda$ can lead to slow convergence and increase the computational burden of repair layer.

\subsection{Computational Remarks} \label{sec:compute_rmks}

The primary limitation of {\ournet},
shared by all extant repair layers in the literature, is the computational cost of repair:
applying $J^{\dagger}_{\lambda}(y^k)$ with a direct (factor-solve) method costs $\mathcal{O}(nm^2 + m^3)$ (for $m<n$)
using the push-through identity $[J^{\top} J + \lambda I]^{-1} J^{\top} = J^{\top} [J J^{\top} + \lambda I]^{-1}$.
However, compared to competitors, {\ournet} strikes a balance between efficient repair and computational overhead.
Cheaper methods like {\dc} fail to enforce constraints consistently across instances,
while existing robust methods like {\optnet} require significantly longer training time and work only on convex constraints.
Our experiments show that {\ournet} consistently ensures a medium-to-high feasibility tolerance, even for \emph{non-convex} constraints, while training faster.
{\ournet} is more efficient on less constrained problems (\ie $m \ll n$) as it uses fewer repair iterations (see \Cref{app:scaling}). 

As an alternative to direct inversion, iterative methods like conjugate gradient can find an approximate solution for high-dimensional problems using cheaper iterations, but often require more iterations to converge and may even diverge.
These methods do not form the Jacobian, but require only vector-Jacobian products (VJP), computed efficiently using automatic differentiation.
Empirically, we observe divergence of {\ournet} with iterative methods, so our experiments use a direct solver.

\section{Experiments} \label{sec:exp}
We demonstrate the effectiveness of {\ournet} on
optimization learning and neural control policies\footnote{Code is available at \url{https://github.com/miniyachi/SnareNet}}.
Experiments run on a server with two 64-core AMD EPYC 7763 @ 2.45 GHz, 1 TB RAM, and NVIDIA A6000 GPUs unless otherwise specified.

\subsection{Optimization Learning} \label{sec:opt-learning}

\emph{Optimization learning} \cite{van2025optimization} task seeks a fast surrogate neural solver for a family of optimization problems parametrized by
input $x \in \mcal{X}$:
\begin{equation} \label{eq:opt-prob}
\min_{y \in \mathbb{R}^n} f_x(y) ~\st~  \ell_x \leq g_x(y) \leq u_x,
\end{equation}
Problems of the form \eqref{eq:opt-prob} can be non-linear and non-convex, and can be slow to solve with traditional optimization solvers.
The goal of optimization learning is to learn a model $\mcal{M}_{\theta}$ that approximates the solution map $\Phi: \mcal{X} \rightarrow \mbb{R}^n$, which maps the instance parameter $x$ to optimal solution $y^{\star} = \Phi(x)$ of \eqref{eq:opt-prob}.
In particular, $\mcal{M}_{\theta}$ must produce a feasible solution for all $x \in \mcal{X}$.

\paragraph{Families of Optimization Problems.}
We consider three families: non-convex linearly-constrained programs (NCLPs) and convex/non-convex quadratically constrained quadratic programs (QCQPs).
\begin{equation*}
    \begin{array}{ccl}
    \textbf{(NCLP)} &\displaystyle\min_{y \in \mathbb{R}^n} & \frac{1}{2}y^TQy+p^T \sin(y)\\
    &\text{s.t.} & Ay \leq b, ~Cy = x, \\
    & &
    \end{array}
    \hfill
    \begin{array}{ccl}
    \textbf{(QCQP)} &\displaystyle\min_{y \in \mathbb{R}^n} & \frac{1}{2}y^TQy+p^Ty\\
        & \text{s.t.} & y^TH_iy + g_i^Ty \leq h_i, ~ i = 1, \ldots, m_{\text{ineq}}\\
        & & Cy = x,
    \end{array}
\end{equation*}
where $x \in \mbb{R}^{m_{\text{eq}}}$ is the input and the other problem data are constant within the family.
Each family consists of $10000$ problem instances and is split into train/valid/test set in the ratio 8:1:1.


\paragraph{Baselines.}
We compare {\ournet} to four baselines: {\optnet} \cite{amos2017optnet}, {\dc} \cite{donti2021dc3}, and {\hardnet} \cite{min2024hardnet} achieve feasibility using end-to-end trainable repair layers; and {\hproj} \cite{liang2024homeomorphic} uses post-processing with a learning-based projection module.
Further details of each baseline appear in \Cref{app:baselines}.




\paragraph{Evaluation.} \Cref{tab:eval-metrics} summarizes the six evaluation metrics in our experiments. Each subpart in \Cref{fig:test-bar} follows the same 2-by-3 layout.
All experiments are run for 5 random seeds, over which the metrics are averaged.
The optimality gap is assessed relative to SciPy's SLSQP solver \cite{2020SciPy-NMeth} on NCLPs and Gurobi \cite{gurobi} on convex QCQPs.
We evaluate optimality by objective values on non-convex QCQPs due to long traditional solver time.
The complete test results are provided in \Cref{app:tables}.

\begin{table}[h]
\centering
\resizebox{0.9\linewidth}{!}{%
\begin{tabular}{c|c|c|c}
 & \bf{Optimality Gap / Objective} & \bf{Inequality Violation} & \bf{Equality Violation} \\ \hline
\rule{0pt}{4ex}\bf{Geometric Mean} & $\displaystyle\gmean_{x \in \mathcal{X}_{\test}} (f_x(\hat{y}) - f_x^\star)$ / $\displaystyle\mean_{x \in \mathcal{X}_{\test}} f_x(\hat{y})$ & \multicolumn{2}{c}{$\displaystyle\gmean_{x \in \mathcal{X}_{\test}} \gmean_{j \in [n_{\ineq}]} (j\text{-th inequality/equality violation})$} \\
\hline
\rule{0pt}{4ex}\bf{Maximum} & $\displaystyle\max_{x \in \mathcal{X}_{\test}} (f_x(\hat{y}) - f_x^\star)$ / $\displaystyle\max_{x \in \mathcal{X}_{\test}} f_x(\hat{y})$ &
\multicolumn{2}{c}{$\displaystyle\max_{x \in \mathcal{X}_{\test}} \max_{j \in [n_{\ineq}]} (j\text{-th inequality/equality violation})$} \\
\end{tabular}%
}\vspace*{0.5em}
\caption{Six evaluation metrics where $\gmean(\cdot)$ denotes the geometric mean.}
\label{tab:eval-metrics}
\end{table}


\begin{figure*}
    \centering
    \includegraphics[width=0.495\linewidth]{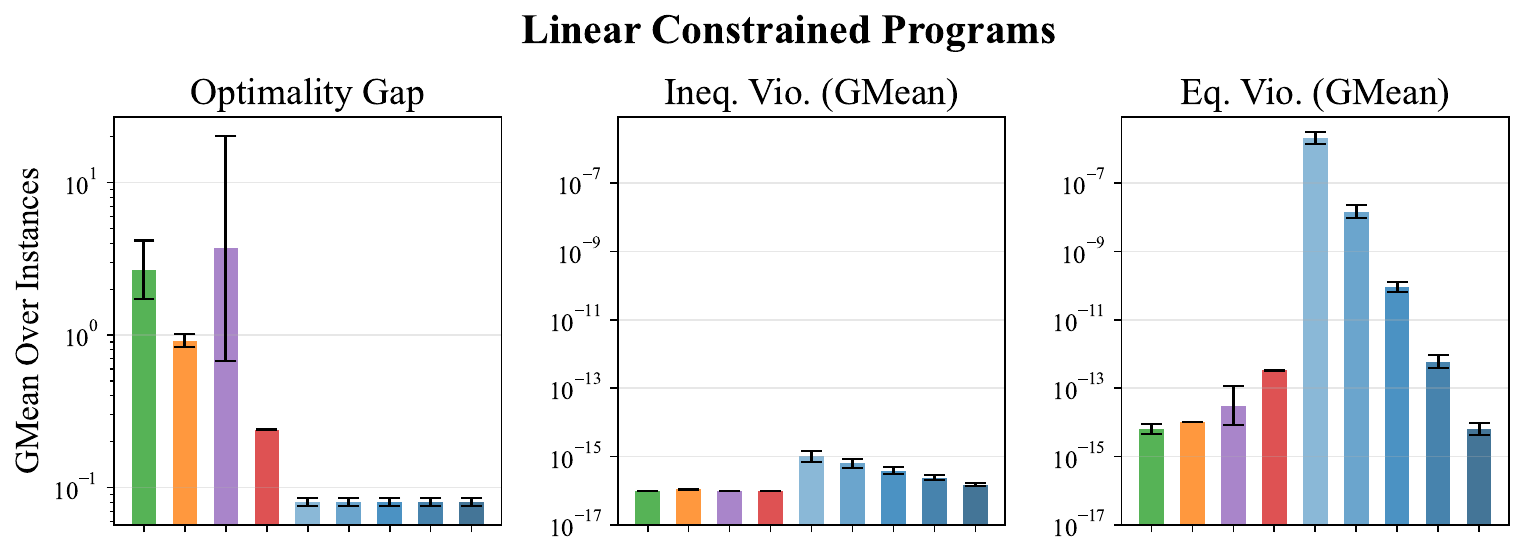}
    \includegraphics[width=0.495\linewidth]{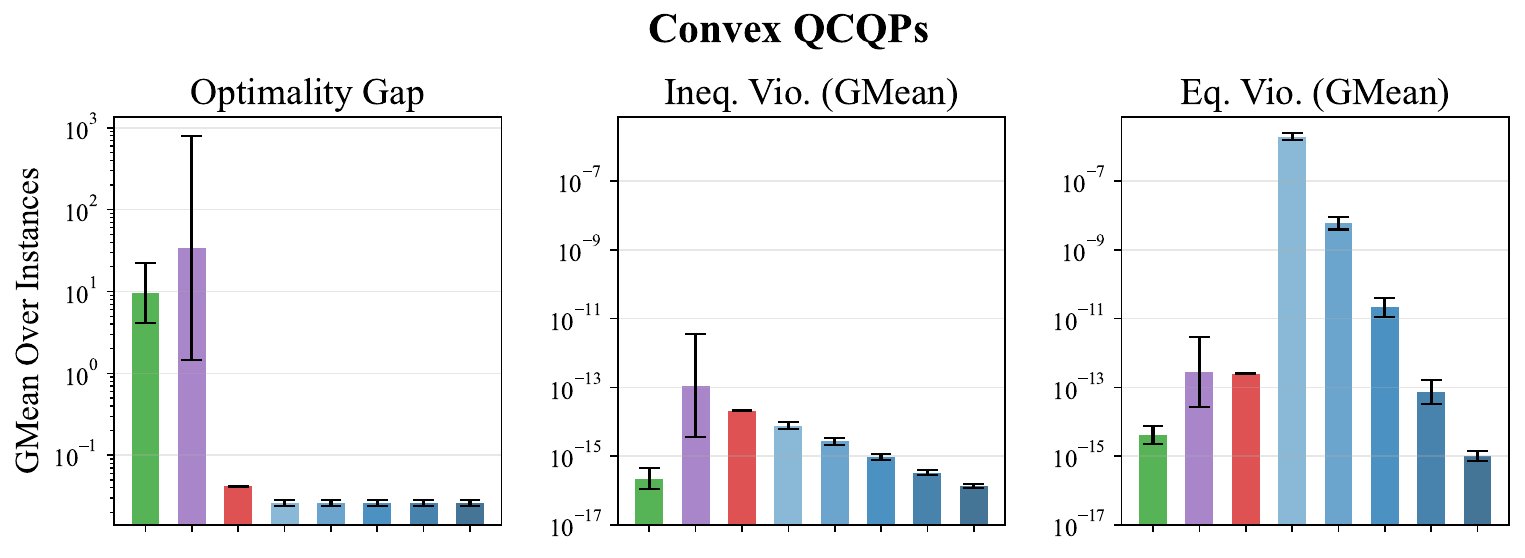}
    \includegraphics[width=0.495\linewidth]{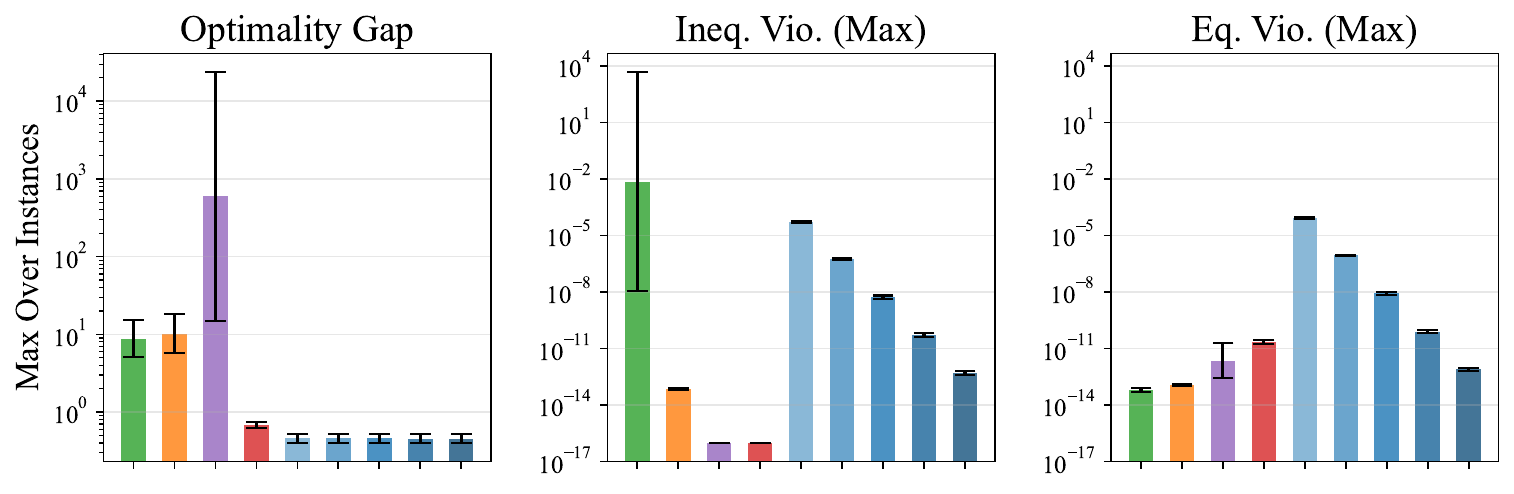}
    \includegraphics[width=0.495\linewidth]{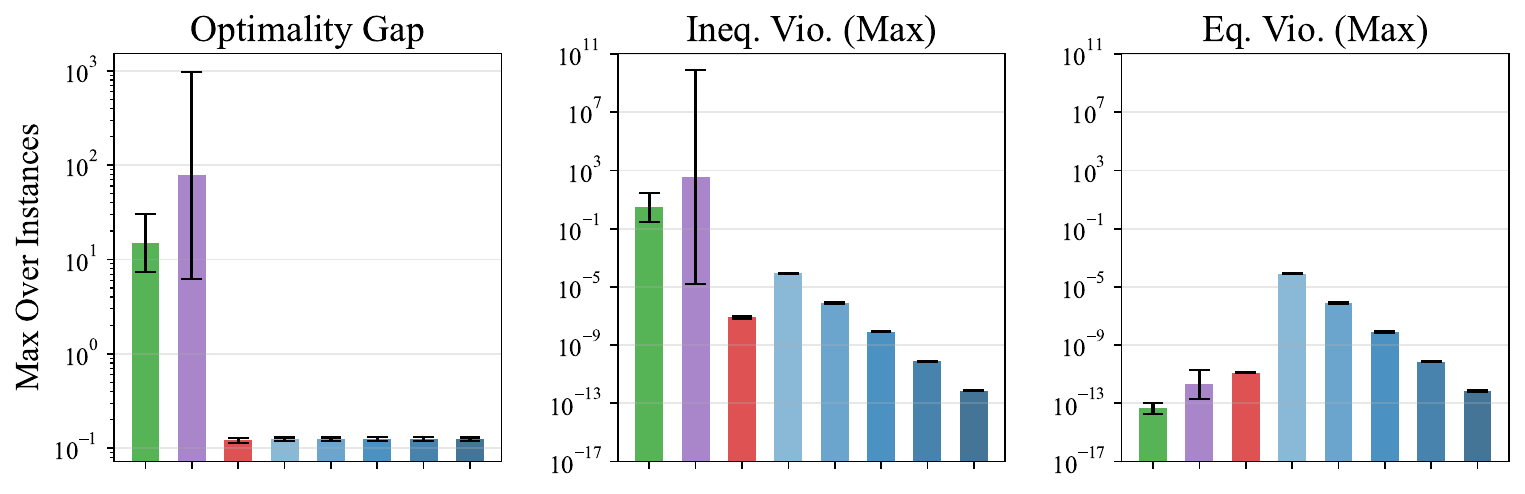}
    \begin{minipage}[c]{0.495\linewidth}
        \includegraphics[width=\linewidth]{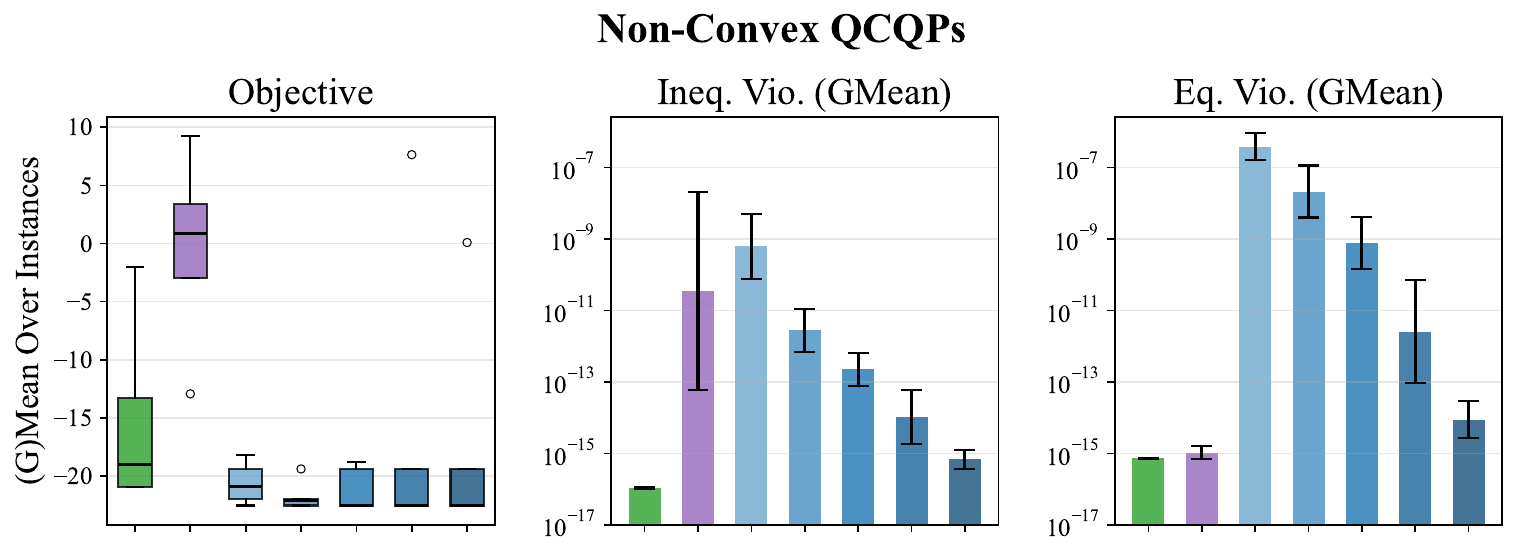}\\
        \includegraphics[width=\linewidth]{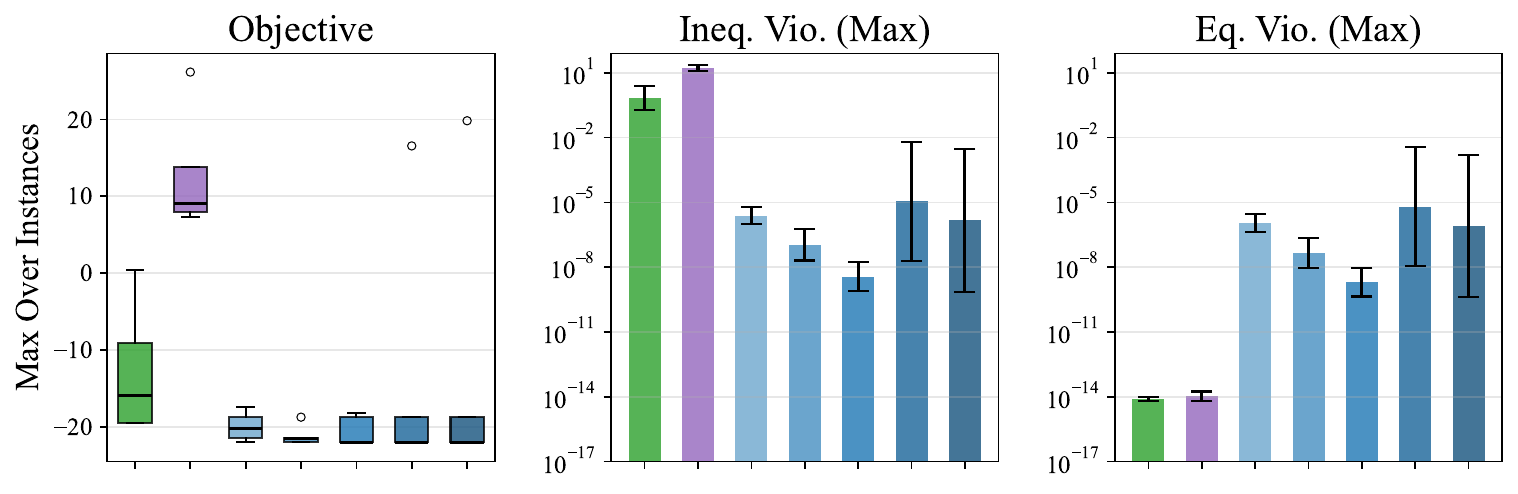}
    \end{minipage}%
    \begin{minipage}[c]{0.495\linewidth}
        \centering
        \includegraphics[width=0.92\linewidth]{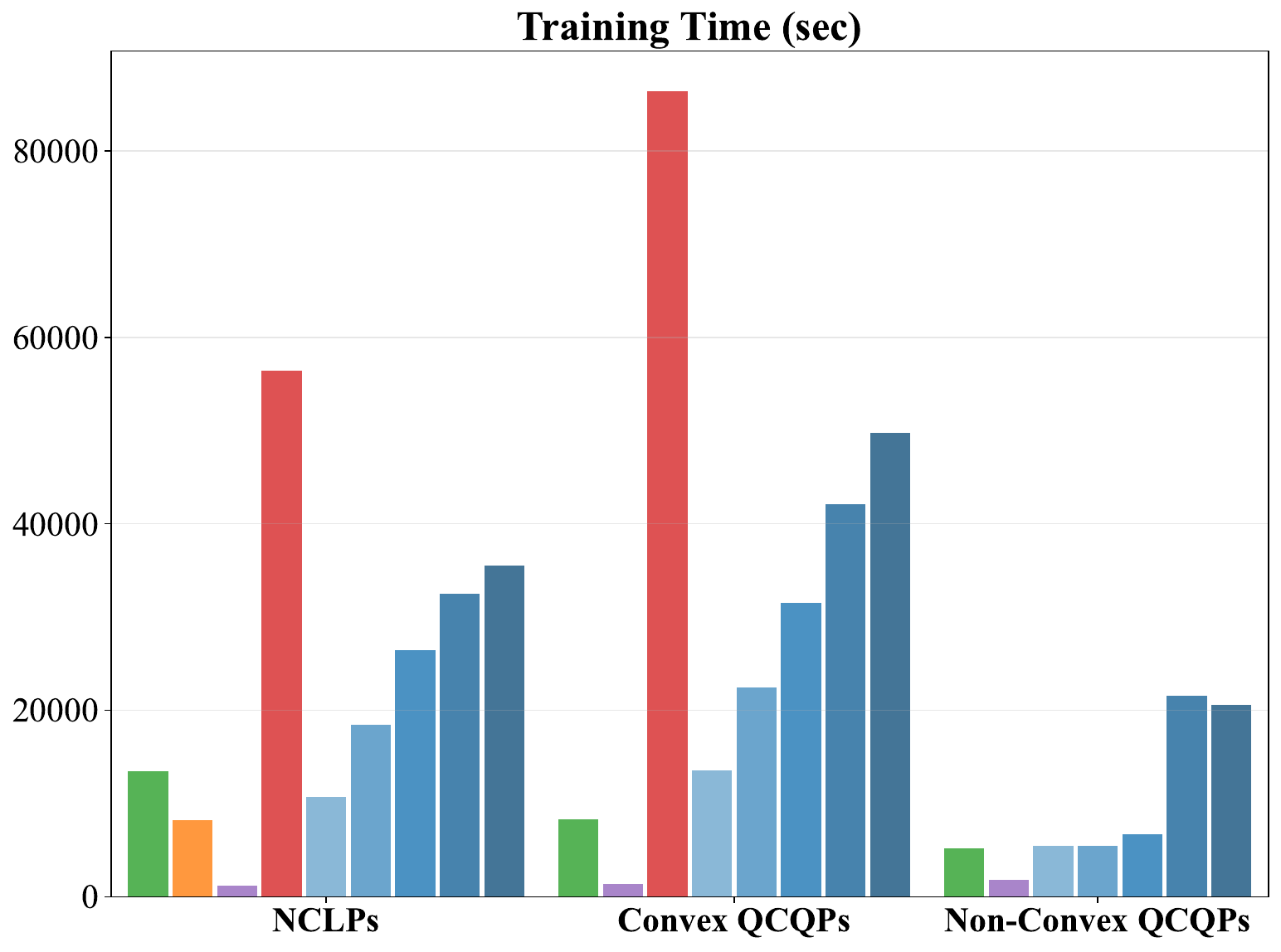}
    \end{minipage}
    \includegraphics[width=0.9\linewidth]{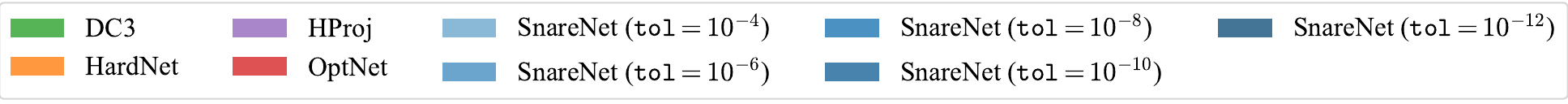}
    \caption{Evaluation metrics on 833 test instances of NCLPs and Convex/Non-Convex QCQPs and the average training time. Black error bars indicate the standard deviation across $5$ random seeds.}
    \label{fig:test-bar}
\end{figure*}

\paragraph{Feasibility Control and Low Variation.}
\Cref{fig:test-bar} evaluates trained neural solvers on a separate test set for all three families\footnote{Each run for NCLPs/convex QCQPs/non-convex QCQPs was terminated after 1500/1500/2000 epochs or 16/24/24 hours.}.
{\dc} and {\hproj} show larger seed-to-seed variability in inequality-constraint violations and optimality gap
and may fail to ensure feasibility across all test instances, especially for QCQPs.
{\hardnet} exhibits larger variability in maximum optimality gap.
Our {\ournet} shows less variability and successfully controls feasibility up to the specified tolerance.

\paragraph{Better Optimality.}
\Cref{fig:test-bar} also shows that {\ournet} achieves optimality gaps at least one order of magnitude smaller than all baseline methods except for {\optnet} while maintaining feasibility within the prescribed feasibility tolerance $\mathtt{tol}$.
However, {\optnet} does not apply to non-convex constraints and requires much longer training to achieve similar performance on convex constraints.
{\hproj} exhibits substantial variation in optimality gap across random seeds\footnote{The results of {\hproj} on non-convex QCQPs are averaged over 3 of 5 random seeds (2 were excluded due to numerical instability and outlier objective values of $\approx 5000$, respectively.}because its projection step is applied post-training, which can degrade optimality.

\paragraph{$\lambda$ Improves Optimality Even for Linear Constraints.}
{\ournet} specializes to {\hardnet} when $\lambda = 0$.
Larger $\lambda > 0$ requires more iterations to enforce feasibility but can improve the optimality gap by orders of magnitude. See \Cref{fig:noncvx-lambda}.


\paragraph{{\ournet} Handles Many Constraints.}
{\hardnet} requires full row rank linear constraints and is thus limited to at most $2n$ constraints.
{\ournet} overcomes this limitation through LM-regularization.
\Cref{fig:cvx_qcqp_ineq} presents test results on convex QCQPs with
$n_{\ineq} \in \{10, 50, 100\}$ inequality constraints.
{\ournet} consistently produces feasible solutions within tolerance $10^{-4}$ across all test instances.
In contrast, {\dc} achieves feasibility on only $80\%$ of test instances on average and exhibits high sensitivity to random initialization.
{\hproj} shows substantial variation in feasibility rates for problems with $n_{\ineq} = 50$ and $100$ inequality constraints.
Inference time for the classical solver Gurobi is significantly slower than neural solvers and scales rapidly with the number of inequality constraints.


\begin{figure}
    \centering
    \begin{minipage}[t]{0.43\textwidth}
        \centering
        \includegraphics[width=\linewidth]{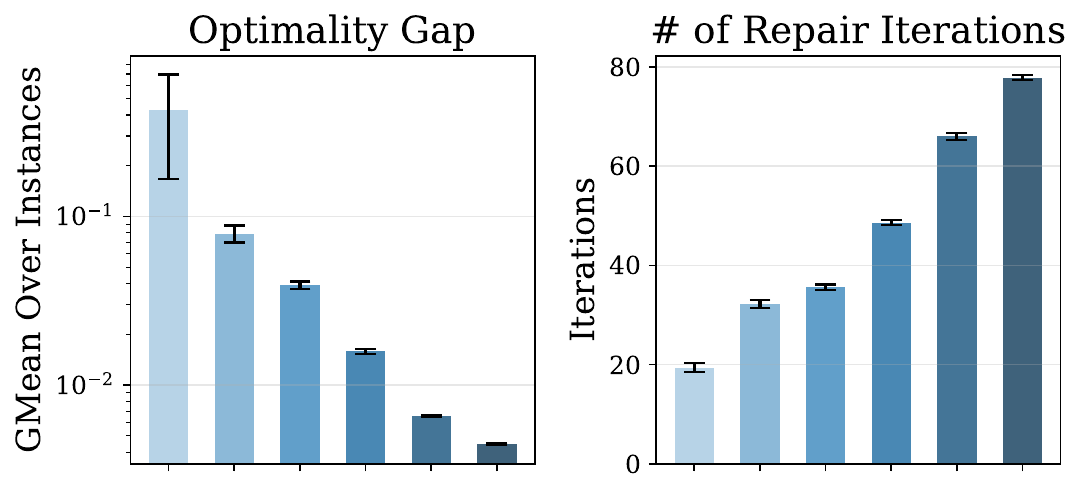}
        \includegraphics[width=\linewidth]{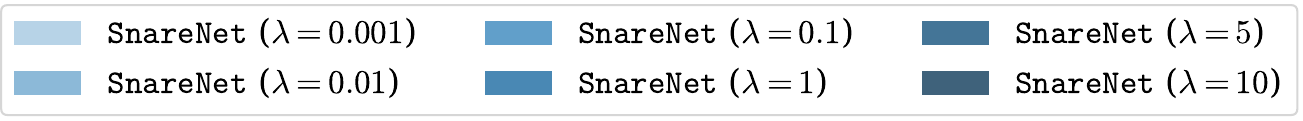}
        \caption{Optimality gap and repair iterations on $833$ test instances of NCLPs.
        }
        \label{fig:noncvx-lambda}
    \end{minipage}
    \hfill
    \begin{minipage}[t]{0.54\textwidth}
        \centering
        \includegraphics[width=\linewidth]{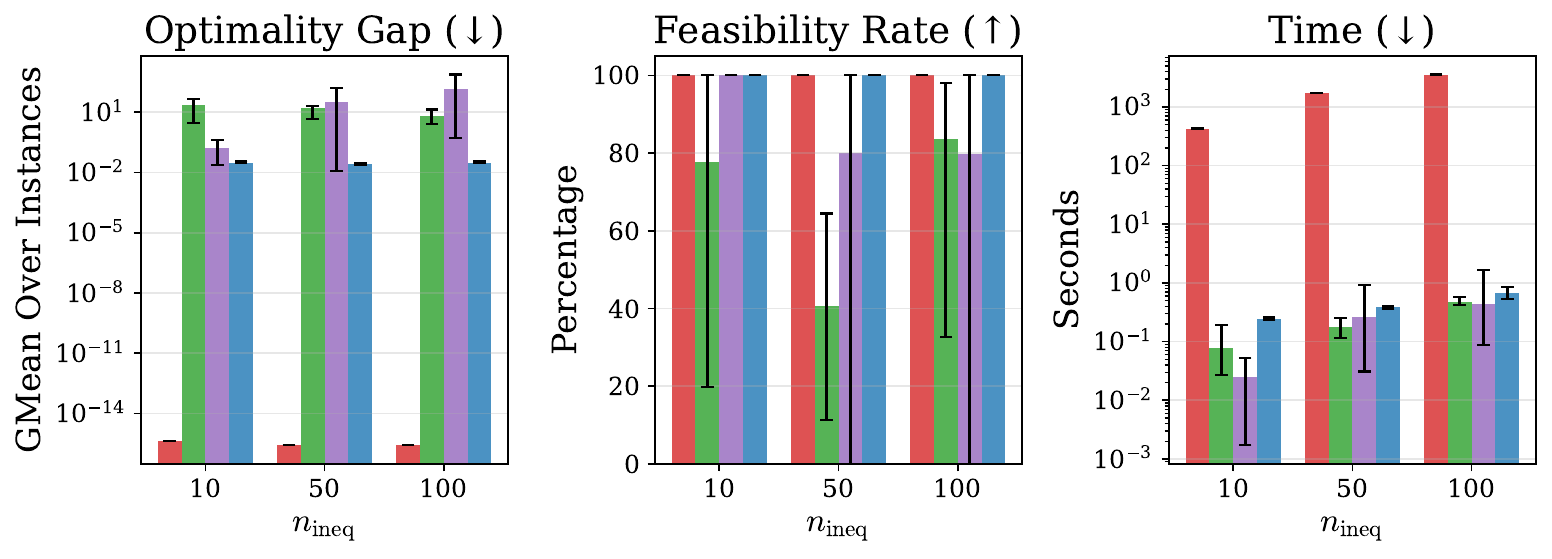}
        \includegraphics[width=\linewidth]{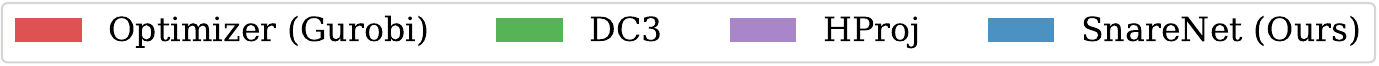}
        \caption{Performance on convex QCQPs with various number of inequality constraints.
        }
        \label{fig:cvx_qcqp_ineq}
    \end{minipage}
\end{figure}

\paragraph{Adaptive Relaxation beats Soft Epochs.}
\Cref{tab:ablation_noncvx_warmup} summarizes an ablation study on the warm-up strategy used in {\hardnet} and {\ournet} with different $\lambda$ values.
Each model was trained for 1000 epochs in total, with the warm-up strategy applied for the first 500 epochs.
The feasibility tolerance of {\ournet} was set to $10^{-8}$ for computational efficiency.
We make three observations:
1) Warm-up strategy is required to achieve better optimality.
2) Adaptive relaxation ({\adarel}) reduces both the maximum and geometric mean of the optimality gap compared to other warm up strategies.
3) The improvement of {\adarel} over soft epochs is more pronounced for smaller $\lambda$ (\eg {\hardnet} $\lambda=0$).

\begin{table}[h]
\centering
\resizebox{\linewidth}{!}{
\begingroup
\begin{tabular}{llcccccc}
\toprule
Method & Warm-up & Max Opt. Gap. & GMean Opt. Gap. & Max Ineq. Error & Gmean Ineq. Error & Max Eq. Error & Gmean Eq. Error \\
\midrule
{\hardnet} (i.e. $\lambda=0$)  & None & $(1.66 \pm 0.35) \times 10^{1}$ & $8.61 \pm 1.62$ & $(2.14 \pm 0.39) \times 10^{-13}$ & $(2.72 \pm 0.53) \times 10^{-16}$ & $(2.53 \pm 0.71) \times 10^{-13}$ & $(2.22 \pm 0.57) \times 10^{-14}$ \\
{\hardnet} (i.e. $\lambda=0$) & 500 Soft Epochs & $9.34 \pm 3.40$ & $1.73 \pm 0.58$ & $(5.88 \pm 1.42) \times 10^{-14}$ & $(1.03 \pm 0.00) \times 10^{-16}$ & $(9.55 \pm 1.07) \times 10^{-14}$ & $(8.50 \pm 0.39) \times 10^{-15}$ \\
{\hardnet} (i.e. $\lambda=0$) & 500 {\adarel} & $9.29 \pm 8.41$ & $1.01 \pm 0.25$ & $(5.41 \pm 1.10) \times 10^{-14}$ & $(1.04 \pm 0.00) \times 10^{-16}$ & $(9.04 \pm 0.99) \times 10^{-14}$ & $(8.54 \pm 0.28) \times 10^{-15}$ \\
\midrule
{\ournet} ($\lambda=0.001$) & None & $9.52 \pm 3.34$ & $6.15 \pm 1.75$ & $(5.67 \pm 1.08) \times 10^{-9}$ & $(2.09 \pm 1.67) \times 10^{-15}$ & $(8.23 \pm 1.24) \times 10^{-9}$ & $(3.24 \pm 1.00) \times 10^{-10}$ \\
{\ournet} ($\lambda=0.001$) & 500 Soft Epochs & $3.25 \pm 1.55$ & $(5.84 \pm 1.93) \times 10^{-1}$ & $(2.05 \pm 0.28) \times 10^{-9}$ & $(1.22 \pm 0.03) \times 10^{-16}$ & $(8.14 \pm 1.14) \times 10^{-9}$ & $(3.59 \pm 3.57) \times 10^{-10}$ \\
{\ournet} ($\lambda=0.001$) & 500 {\adarel} & $2.23 \pm 1.03$ & $(3.57 \pm 1.01) \times 10^{-1}$ & $(3.36 \pm 2.20) \times 10^{-9}$ & $(1.23 \pm 0.04) \times 10^{-16}$ & $(6.70 \pm 2.17) \times 10^{-9}$ & $(4.08 \pm 1.15) \times 10^{-10}$ \\
\midrule
{\ournet} ($\lambda=0.1$) & None & $5.63 \pm 0.49$ & $4.49 \pm 0.53$ & $(1.04 \pm 0.63) \times 10^{-8}$ & $(2.90 \pm 0.80) \times 10^{-15}$ & $(2.00 \pm 0.61) \times 10^{-8}$ & $(2.26 \pm 1.04) \times 10^{-9}$ \\
{\ournet} ($\lambda=0.1$) & 500 Soft Epochs & $(1.90 \pm 0.27) \times 10^{-1}$ & $(3.91 \pm 0.44) \times 10^{-2}$ & $(1.46 \pm 0.21) \times 10^{-8}$ & $(3.23 \pm 0.16) \times 10^{-16}$ & $(2.16 \pm 0.31) \times 10^{-8}$ & $(4.10 \pm 0.94) \times 10^{-10}$ \\
{\ournet} ($\lambda=0.1$) & 500 {\adarel} & $(2.11 \pm 0.25) \times 10^{-1}$ & $(3.79 \pm 0.28) \times 10^{-2}$ & $(1.39 \pm 0.37) \times 10^{-8}$ & $(3.30 \pm 0.13) \times 10^{-16}$ & $(2.08 \pm 0.51) \times 10^{-8}$ & $(4.08 \pm 1.26) \times 10^{-10}$ \\

\bottomrule
\end{tabular}
\endgroup}
\vspace*{0.5em}
\caption{Ablation on warm-up strategies for NCLPs. Values are mean $\pm$ std across $5$ random seeds.}
\label{tab:ablation_noncvx_warmup}
\end{table}

\subsection{Neural Control Policies}
\begin{minipage}[t]{0.68\textwidth}
Neural control policies use deep neural networks to learn mappings from system states $x$ to control actions $u$, either as standalone policies or to enhance traditional control methods.
In this experiment, we train a neural network policy $\pi_{\theta}(x)$ to control a unicycle system, avoiding obstacle collisions by enforcing safety constraints \cite{min2024hardnet,tayal2024collisionconeapproachcontrol}.
The neural policy is constructed as the sum of a nominal controller $\pi_{\nom}(x)$, designed without obstacle awareness, and a learned correction term $\mcal{M}_{\theta}(x)$:
$\pi_{\theta}(x) := \pi_{\nom}(x) + \mcal{M}_{\theta}(x)$.
We use \emph{control barrier functions} (CBFs) \cite{ames2019controlbarrierfunctionstheory} defined in \Cref{app:cbf} to parametrize the safe set.
Collision-free trajectories are guaranteed by enforcing
\begin{equation*} 
    \nabla h_j(x)^T(F(x) + G \pi_{\theta}(x) ) \geq -\alpha h_j(x)
\end{equation*}
for each obstacle with CBF $h_j$ and some $\alpha > 0$.
The number of constraints scales linearly with the number of obstacles.
{\hardnet} can avoid at most two obstacles as it requires full row rank constraints and the system has only two controls, whereas {\ournet} handles arbitrarily many obstacles.
\Cref{fig:neural_traj} shows the learned neural policies and the detailed evaluation statistics can be found in \Cref{app:cbf}.
Across all test instances, {\ournet} achieves a 99.4\% feasibility rate, compared to {\dc}'s 64.9\% feasibility rate, demonstrating the effectiveness of {\ournet} in enforcing safety constraints in neural control policies.
Since we use a higher order CBF, a trajectory may violate a constraint without visually colliding with any obstacle, \eg DC3 trajectory in \Cref{fig:seed8}.
\end{minipage}
\hfill
\begin{minipage}[t]{0.28\textwidth}
    \vspace*{-3.25em}
    \captionsetup{type=figure}
    \centering
    \begin{subfigure}[b]{\linewidth}
        \centering
        \caption{Initialization box 1}
        \includegraphics[width=\textwidth]{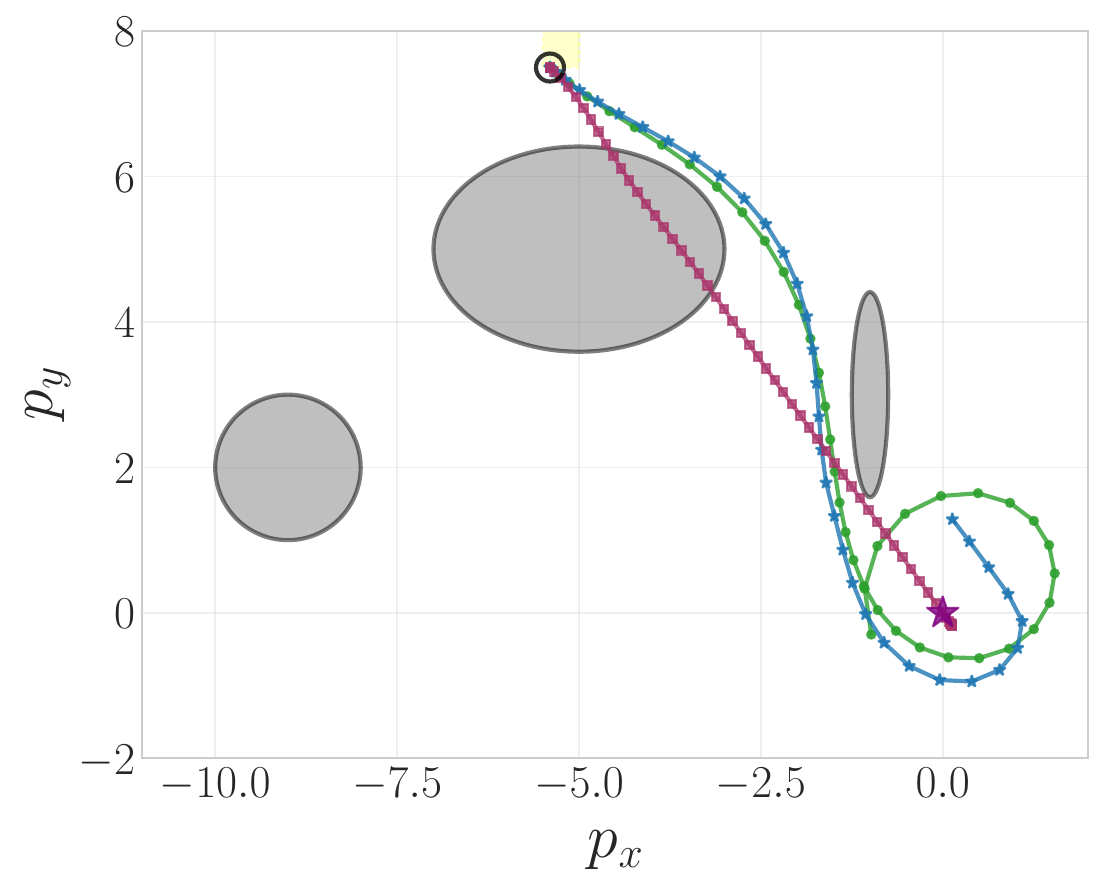}
        \label{fig:seed8}
    \end{subfigure}
    \begin{subfigure}[b]{\linewidth}
        \centering
        \caption{Initialization box 2}
        \includegraphics[width=\textwidth]{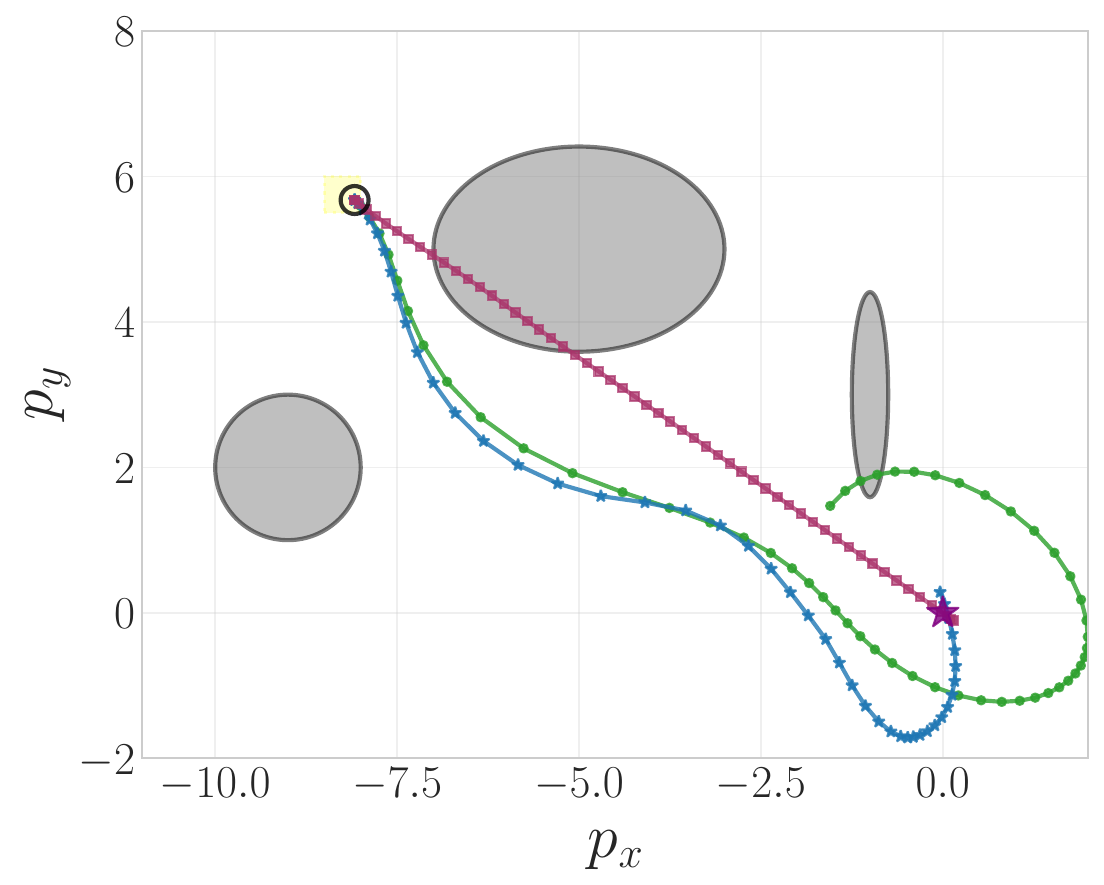}
        \label{fig:seed9}
    \end{subfigure}
    \begin{subfigure}[b]{\linewidth}
        \centering
        \includegraphics[width=\linewidth]{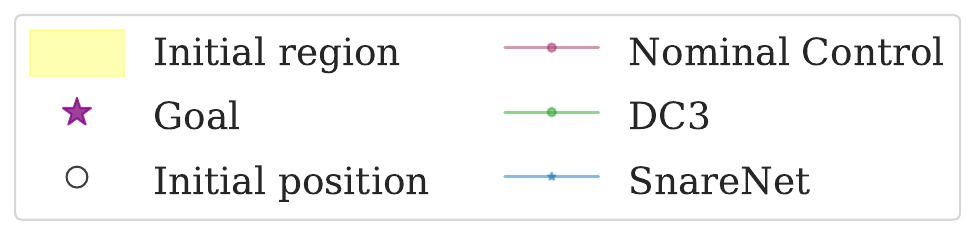}
    \end{subfigure}
    \caption{Simulated trajectories from neural policies.}
    \label{fig:neural_traj}
\end{minipage}%

\section{Conclusion} \label{sec:conclusion}
We introduce {\ournet}, a framework to enforce input-dependent constraints on the output of neural networks with feasibility guarantees.
Our empirical evaluation demonstrates that {\ournet} consistently produces feasible solutions with better objective values, even for \emph{non-convex} constraints.
Our key innovations include new perspective to reformulate an inequality feasibility problem as a more tractable equation solving problem,
regularized Newton updates for efficient and stable convergence,
and adaptive relaxation that progressively tightens the constraints.

\bibliography{ref}
\bibliographystyle{plainnat}

\appendix
\section{Convergence Analysis of {\ournet}} \label{app:proofs}
The convergence analysis of {\ournet} is based on the observation that 
{\ournet}'s repair layer update with $\lambda > 0$ is equivalent to preconditioned gradient descent (PGD) on the constraint violation function defined as
\begin{equation} \label{eq:F}
F(y) = \frac{1}{2} \| g(y) - \mathcal{P}_{\mathbf{B}(\ell, u)}(g(y)) \|^2.
\end{equation}
\Cref{sec:F-property} establishes the analytical properties of $F$ that are crucial for the convergence analysis. The convergence analysis for linear and convex constraints follows in \Cref{sec:pf-linear} and \Cref{sec:pf-cvx}, respectively.
\subsection{Analytical Properties of $F$} \label{sec:F-property}
The function $F$ is twice differentiable almost everywhere, despite the presence of the seemingly non-smooth projection operator $\mathcal{P}_{\mathbf{B}(\ell, u)}$.
Moreover, when $g(y) = Ay$ is linear with $A \in \mathbb{R}^{m \times n}$, the function $F$ is convex. See \Cref{prop:F-property}.

\begin{proposition}[Differentiability \& Convexity] \label{prop:F-property}
Let $g : \mathbb{R}^n
\rightarrow \mathbb{R}^m$ be a twice differentiable function. Then $F
(y)$ defined in \eqref{eq:F} is twice differentiable almost everywhere and its gradient is given by
\begin{equation} \label{eq:grad-F}
\nabla F (y) = J_g (y)^{\top} (g (y) -\mathcal{P}_{\mathbf{B} (\ell, u)}
   (g (y))).
\end{equation}
In particular, if $g = A \in \mathbb{R}^{m \times n}$ is linear, the constraint violation function $F (y)$ is convex.
\end{proposition}
\begin{proof}
Observe that $F(y) = R(g(y))$ a composition of $g$ and a separable function $R: \mathbb{R}^m \rightarrow \mathbb{R}$ defined by
\begin{equation} \label{eq:R}
    R(z) := \frac{1}{2} \| z - \mathcal{P}_{\mathbf{B}(\ell, u)}(z) \|^2 
    = \frac{1}{2} \sum_{i=1}^m (z_i - \mathcal{P}_{[\ell_i, u_i]} (z_i))^2,
\end{equation}
where $\mathcal{P}_{[\ell_i, u_i]}$ being the one-dimensional projection operator onto the interval $[\ell_i, u_i]$.
\begin{itemize}
    \item \textbf{Almost everywhere twice differentiability.} 
    It suffices to show that $r_{[\ell, u]}(a) := \frac{1}{2}(a - \mathcal{P}_{[\ell, u]} (a))^2$ with $a \in \mathbb{R}$ is twice differentiable almost everywhere since $R(z) = \sum_{i=1}^m r_{[\ell_i, u_i]}(z_i)$ is the sum of such functions and $F(y) = R(g(y))$ is a composition of $R$ and a twice differentiable $g$.
    The piecewise definition of $r(a)$ and its derivatives $r' (a)$ and $r'' (a)$ are
    \begin{align*}
    r_{[\ell, u]}(a) &= \begin{cases}\begin{array}{ll}
        \frac{1}{2} (a - \ell)^2, & \text{if } a < \ell ;\\
        0, & \text{if } \ell \leq a \leq u ;\\
        \frac{1}{2} (a - u)^2, & \text{if } a > u,
    \end{array} \end{cases}
    r_{[\ell, u]}' (a) = \begin{cases} \begin{array}{ll}
        a - \ell, & \text{if } a < \ell ;\\
        0, & \text{if } \ell \leq a \leq u ;\\
        a - u, & \text{if } a > u,
    \end{array} \end{cases} \\
    r_{[\ell, u]}'' (a) &= \begin{cases} \begin{array}{ll}
        0, & \text{if } \ell < a < u ;\\
        1, & \text{if } a < \ell \text{ or } a > u.
        \end{array} \end{cases}
    \end{align*}
    Note that $r_{[\ell, u]}' (a) = a - \mathcal{P}_{[\ell, u]} (a)$ is continuous everywhere,
    and the second derivative $r_{[\ell, u]}'' (a)$ is well-defined for all $a \in \mathbb{R}$ except at the two points $a = \ell$ and $a = u$.
    Thus, $r_{[\ell, u]}(a)$ is twice differentiable almost everywhere, and so is $F(y)$.

    \item \textbf{Gradient expression.} 
    Note that $\nabla R(z) = z - \mathcal{P}_{\mathbf{B}(\ell, u)}(z)$ since $r_{[\ell, u]}' (a) = a - \mathcal{P}_{[\ell, u]} (a)$. 
    The gradient of $F(y)$ follows by the chain rule:
    \begin{equation*} 
    \nabla F (y) = J_g (y)^{\top} \nabla R(g(y)) = J_g (y)^{\top} (g (y) -\mathcal{P}_{\mathbf{B} (\ell, u)}
    (g (y))).
    \end{equation*}

    \item \textbf{Convexity of $F$ when $g$ is linear.}
    Observe that $R(z) = \sum_{i=1}^m r_{[\ell_i, u_i]}(z_i)$ is a sum of convex functions since each $r_{[\ell_i, u_i]}$ is convex. 
    Thus, $R$ is convex. As $F(y) = R(Ay)$ is a composition of a convex function $R$ and a linear function $Ay$, $F$ is convex.
\end{itemize}
\end{proof}

Using expression \eqref{eq:grad-F} of $\nabla F(y)$, {\ournet}'s repair update \eqref{eq:LM-update} with $\lambda > 0$ can be expressed as
\begin{align}
  y^{k + 1} & = y^k - (J_g (y^k)^{\top} J_g (y^k) + \lambda I)^{- 1} J_g (y^k)^{\top}
  (g (y^k) -\mathcal{P}_{\mathbf{B} (\ell, u)} (g (y^k))) \nonumber\\
  &= y^k - (J_g (y^k)^{\top} J_g (y^k) + \lambda I)^{- 1} \nabla
  F (y^k). \label{eq:pgd}
\end{align}
\Cref{eq:pgd} can be interpreted as a PGD step 
on $F (y)$ with preconditioner $P_k := (J_g (y^k)^{\top} J_g (y^k) + \lambda I)^{- 1}$.
To show the convergence under the choice $P_k$, we first establish an upper bound on $F(y)$ in \Cref{lem:F-upper-bd} that is analogous to the quadratic upper bound for smooth functions, and then show that $F$ satisfies the Polyak-Lojasiewicz (PL) condition under the preconditioned norm when $J_g(y)$ has full-row rank in \Cref{lem:F-PL}. 
Finally, we prove the convergence of {\ournet} in \Cref{thm:linear-main} and \Cref{thm:nonlinear-main} for linear and nonlinear constraints, respectively.

\begin{lemma}[Upper bound on $F$] \label{lem:F-upper-bd}
    Let $g : \mathbb{R}^n \rightarrow \mathbb{R}^m$ be a twice
    differentiable function with an $L_g$-Lipschitz continuous Jacobian. Then for
    any $x, y \in \mathbb{R}^n$, the constraint violation function $F (y)$ defined in \eqref{eq:F} satisfies the bound
    \begin{equation} \label{eq:F-upper-bd}
    \begin{aligned}
    F (y) \leq{} & F (x) + \langle \nabla F (x), y - x \rangle + \frac{1}{2} (y
    - x)^{\top} J_g (x)^{\top} J_g (x) (y - x)\\
    & + \frac{L_g}{2} \left[ \| \nabla R (g (x)) \| + \| J_g (x) (y - x) \|
    + \frac{L_g}{4} \| y - x \|^2 \right] \| y - x \|^2.
    \end{aligned}
    \end{equation}
    In particular, if $g(y) = A y$ is linear with $A \in \mathbb{R}^{m \times n}$, the bound in \eqref{eq:F-upper-bd} reduces to
    \begin{equation} \label{eq:F-upper-bd-linear}
        F (y) \leq F (x) + \langle \nabla F (x), y - x \rangle + \frac{1}{2} (y - x)^{\top} A^{\top} A (y - x) .
    \end{equation}
\end{lemma}
\begin{proof} 
    We first show that $R(z)$ defined in \eqref{eq:R} is $1$-smooth, or equivalently, its gradient $\nabla R(z) = z - \mathcal{P}_{\mathbf{B}(\ell, u)}(z)$ is $1$-Lipschitz continuous: 
    for any $z, z' \in \mathbb{R}^m$, we have
    \begin{align*}
    \| \nabla R (z) - \nabla R (z') \|^2 & = \| (z -\mathcal{P}_{\mathbf{B}
    (\ell, u)} (z)) - (z' -\mathcal{P}_{\mathbf{B} (\ell, u)} (z')) \|^2\\
    & = \| z - z' \|^2 - 2 \langle \mathcal{P}_{\mathbf{B} (\ell, u)} (z)
    -\mathcal{P}_{\mathbf{B} (\ell, u)} (z'), z - z' \rangle + \| \mathcal{P}_{\mathbf{B} (\ell, u)} (z) -\mathcal{P}_{\mathbf{B} (\ell,u)} (z') \|^2\\
    & \leq \| z - z' \|^2 - \| \mathcal{P}_{\mathbf{B} (\ell, u)} (z)
    -\mathcal{P}_{\mathbf{B} (\ell, u)} (z') \|^2 \tag{by firmly non-expansiveness of $\mathcal{P}_{\mathbf{B} (\ell, u)}$}\\
    & \leq \| z - z' \|^2.
    \end{align*}
    Apply the $1$-smoothness of $R(z)$ to upper bound $F(y) = R(g(y))$ by
    \begin{align}
    F (y) = R (g (y)) & \leq R (g (x)) + \langle \nabla R (g (x)), g (y) - g (x) \rangle
    + \frac{1}{2} \| g (y) - g (x) \|^2 \nonumber\\
    & = F (x) + \langle \nabla R (g (x)), g (y) - g (x) \rangle + \frac{1}{2} \| g (y) - g (x) \|^2. \label{eq:pf-bd-1}
    \end{align}
    Let $e_x(y) := g (y) - g (x) - J_g (x) (y - x)$ denote the error of
    linear approximation.
    Using the fundamental theorem of calculus and the $L_g$-Lipschitz continuity of $J_g$, the error can be bounded by
    \begin{equation} \label{eq:g-linear-approx-err}
    \| e_x(y) \| = \left\| \int^1_0 [J_g (x + t (y - x)) - J_g (x)] (y - x) \,dt \right\| 
    \leq \int^1_0 L_g t \| y - x \|^2 \,dt
    = \frac{L_g}{2} \| y - x \|^2.
    \end{equation}
    Cauchy-Schwartz inequality and \eqref{eq:g-linear-approx-err} imply
    \begin{align}
    \langle \nabla R (g (x)), g (y) - g (x) \rangle & = \langle \nabla R (g
    (x)), g (y) - g (x) - J_g (x) (y - x) \rangle + \langle \nabla R (g (x)),
    J_g (x) (y - x) \rangle \nonumber\\
    & \leq \| \nabla R (g (x)) \|  \| g (y) - g (x) - J_g (x) (y - x) \| +
    \langle J_g (x)^{\top} \nabla R (g (x)), y - x \rangle \nonumber\\
    & \leq \frac{L_g \| \nabla R (g (x)) \|}{2} \| y - x \|^2 + \langle
    \nabla F (x), y - x \rangle. \label{eq:pf-bd-2}
    \end{align}
    \Cref{eq:g-linear-approx-err} also implies
    \begin{align}
    \| g (y) - g (x) \|^2 & = \| e_x(y) + J_g (x) (y - x) \|^2 \nonumber\\
    & = \| J_g (x) (y - x) \|^2 + 2 \langle J_g (x) (y - x), e_x(y)
    \rangle + \| e_x(y) \|^2 \nonumber\\
    & \leq \frac{1}{2} (y - x)^{\top} J_g (x)^{\top} J_g (x) (y - x) + L_g \| J_g (x) (y - x) \| \| y - x \|^2 +
    \frac{L_g^2}{4} \| y - x \|^4 . \label{eq:pf-bd-3}
    \end{align}
    Combine \eqref{eq:pf-bd-1}, \eqref{eq:pf-bd-2}, and \eqref{eq:pf-bd-3} to conclude \eqref{eq:F-upper-bd}. 
    When $g(y) = A(y)$ is linear, the Jacobian of $g$ is constant $J_g(y) = A$ for all $y \in \mathbb{R}^n$ and thus $L_g = 0$. The bound reduces to \eqref{eq:F-upper-bd-linear}.
\end{proof}

\begin{lemma}[PL-condition in preconditioned norm] \label{lem:F-PL}
    Suppose $J_g (y)$ has full-row rank with minimum singular value $\sigma_{\min} (J_g (y)) > 0$ and $\mcal{C} := \{ y \in \mathbb{R}^n \mid \ell \leq g(y) \leq u \}$ is non-empty. The constraint violation function $F (y)$ defined in \eqref{eq:F} satisfies the Polyak-Lojasiewicz (PL) condition under the norm $\|
\cdot \|_{(J_g (y)^{\top} J_g (y) + \lambda I)^{- 1}}$:
\begin{equation*}
\frac{1}{2} {\| \nabla F (y) \|^2_{(J_g (y)^{\top} J_g (y) + \lambda I)^{-1}}}  \geq \frac{\sigma^2_{\min} (J_g (y))}{\sigma^2_{\min} (J_g (y)) + \lambda} [F (y) - \min_{y \in \mathbb{R}^n} F(y)] = \frac{\sigma^2_{\min} (J_g (y))}{\sigma^2_{\min} (J_g (y)) + \lambda} F (y).
\end{equation*}
\end{lemma}
\begin{proof}
Let $d (y) := g (y) - \mathcal{P}_{\mathbf{B}} (g
(y))$. Since $\nabla F (y) = J_g (y)^{\top} d (y)$, the gradient under preconditioned norm is
\begin{equation*}
    \| \nabla F (y) \|_{(J_g (y)^{\top} J_g (y) + \lambda I)^{- 1}}^2 = d
   (y)^{\top} J_g (y) [J_g (y)^{\top} J_g (y) + \lambda I]^{- 1} J_g
   (y)^{\top} d (y).
\end{equation*}
Let $\sigma_i (J_g (y))$ denote the $i$-th singular value of $J_g (y)$.
Then the singular values of $J_g (y) [J_g (y)^{\top} J_g (y) + \lambda I]^{- 1} J_g (y)^{\top}$ are $\frac{\sigma^2_i (J_g (y))}{\sigma^2_i (J_g (y)) + \lambda}$ for $i = 1, \ldots, m$ with minimum $\frac{\sigma^2_{\min} (J_g (y))}{\sigma^2_{\min} (J_g (y)) + \lambda}$. 
Since $F(y) = \frac{1}{2} \| d (y) \|^2$ and non-empty $\mathcal{C}$ guarantees $\min_{y \in \mathbb{R}^n} F(y) = 0$, we conclude
    \begin{equation*}
        \frac{1}{2} \| \nabla F (y) \|_{(J_g (y)^{\top} J_g (y) + \lambda I)^{-
    1}}^2 \geq \frac{\sigma^2_{\min} (J_g (y))}{2 (\sigma^2_{\min} (J_g (y)) +
    \lambda)} \| d (y) \|^2 = \frac{\sigma^2_{\min} (J_g (y))}{\sigma^2_{\min}
    (J_g (y)) + \lambda} F (y).
    \end{equation*}
\end{proof}

\subsection{Proof of \Cref{thm:linear-main}} \label{sec:pf-linear}

\begin{proof}
Denote $H_{\lambda} := A^{\top} A + \lambda I$. Then $y^{k
+ 1} = y^k - H_{\lambda}^{- 1} \nabla F (y^k)$ and $A^{\top} A \prec
H_{\lambda}$. Substitute $y = y^{k + 1}$ and $x = y^k$ 
into \eqref{eq:F-upper-bd-linear} to obtain
\begin{align}
  F (y^{k + 1}) & \leq F (y^k) + \langle \nabla F (y^k), y^{k + 1} - y^k
  \rangle + \frac{1}{2} (y^{k + 1} - y^k)^{\top} A^{\top} A (y^{k + 1} -
  y^k) \nonumber\\
  & \leq F (y^k) + \langle \nabla F (y^k), y^{k + 1} - y^k \rangle +
  \frac{1}{2} (y^{k + 1} - y^k)^{\top} H_{\lambda} (y^{k + 1} - y^k)\nonumber\\
  & = F (y^k) - \nabla F (y^k)^{\top} H_{\lambda}^{- 1} \nabla F (y^k) +
  \frac{1}{2} \nabla F (y^k)^{\top} H_{\lambda}^{- 1} \nabla F (y^k)\nonumber\\
  & \leq F (y^k) - \frac{1}{2} \| \nabla F (y^k) \|^2_{H_{\lambda}^{- 1}}, \label{eq:descent-linear}
\end{align}
proving the descent property for {\ournet}'s update. Since $A$ has full row rank, 
\Cref{lem:F-PL} guarantees $F (y)$ satisfies PL-condition and thus
\begin{equation*}
    F (y^{k + 1}) \leq F (y^k) - \frac{\sigma^2_{\min}}{\sigma^2_{\min} +
    \lambda} F (y^k) = \left( 1 - \frac{\sigma^2_{\min}}{\sigma^2_{\min} +
    \lambda} \right) F (y^k) .
\end{equation*}
\end{proof}

\subsection{Proof of \Cref{thm:nonlinear-main}} \label{sec:pf-cvx}

\begin{proof}
    Observe that $F$ is convex 
    under the convex constraints represented by $\mathcal{C} := \{ y \in \mathbb{R}^n \mid h(y) \leq 0,\, \ell_A \leq Ay \leq u_A \}$ since $F$ takes the form
    \begin{equation*}
        F (y) := \frac{1}{2} \| \max (0, h (y)) \|^2 
        + \frac{1}{2} \| A y -\mathcal{P}_{\mathbf{B} (\ell_A, u_A)} (A y) \|^2
    \end{equation*}
    and both $\|\max (0, h (y)) \|^2$ and $\| A y -\mathcal{P}_{\mathbf{B} (\ell_A, u_A)} (A y) \|^2$ are convex functions of $y$.
    Using {\ournet}'s notation, the constraint function $g$, the bounds $\ell, u$, and the Jacobian $J_g$ are (defined) as
    \begin{equation*}
        g (y) := \begin{bmatrix}
            h (y)\\
            A y
        \end{bmatrix}, ~
        \ell := \begin{bmatrix}
            - \infty\\
            \ell_A
        \end{bmatrix}, ~
        u := \begin{bmatrix}
            0\\
            u_A
        \end{bmatrix}, ~
        J_g (y) := \begin{bmatrix}
            J_h (y)\\
            A
        \end{bmatrix}.
    \end{equation*}
    The Jacobian of $g$ is $L$-Lipschitz continuous since $J_h (y)$ is $L$-Lipschitz continuous and $A$ is constant.

    Now, we show that our choice of $\lambda$ guarantees $y^k \in \mcal{D}$ for all $k \geq 0$ 
    and also the descent property $F (y^{k + 1}) \leq F (y^k)$.
    We prove by induction. 
    Note that $y^0 \in \mathcal{D}$ by definition. Suppose $y^k \in \mathcal{D}$ by induction hypothesis. 
    Substitute $y = y^{k + 1}$ and $x = y^k$ into \eqref{eq:F-upper-bd} to obtain
\begin{equation} \label{eq:pf-nonlinear-1}
    \begin{aligned}
    F (y^{k + 1}) & \leq F (y^k) + \langle \nabla F (y^k), y^{k + 1} - y^k
    \rangle + \tfrac{1}{2} (y^{k + 1} - y^k)^{\top} J_g (y^k)^{\top} J_g (y^k)
    (y^{k + 1} - y^k)\\
    &\quad + \tfrac{L_g}{2} \left[ \| \nabla R (g (y^k)) \| + \| J_g (y^k) (y^{k
    + 1} - y^k) \| + \tfrac{L_g}{4} \| y^{k + 1} - y^k \|^2 \right] \| y^{k + 1}
    - y^k \|^2 .
    \end{aligned}
\end{equation}
Since $y^k \in \mathcal{D}$, we have $\| \nabla R (g (y^k)) \| = \| g (y^k) - \mcal{P}_{\mbf{B}(\ell, u)}(g (y^k)) \| = \sqrt{2 F(y^k)} \leq \sqrt{2 F(y^0)} =: M$ and thus
\begin{align}
  \| J_g (y^k) (y^{k + 1} - y^k) \| &\leq \left\| J_g (y^k) (J_g (y^k)^{\top} J_g
  (y^k) + \lambda I)^{- 1} J_g (y^k)^{\top} \right\| \left\| \nabla R (g (y^k)) \right\| \leq M, \label{eq:pf-nonlinear-2}\\
  \| y^{k + 1} - y^k \|^2 & = \left\| (J_g (y^k)^{\top} J_g (y^k) + \lambda I)^{-1} J_g (y^k)^{\top} \right\|^2 \left\| \nabla R (g (y^k)) \right\|^2 \leq \tfrac{M^2}{4 \lambda} . \label{eq:pf-nonlinear-3}
\end{align}
Let $H_k := J_g (y^k)^{\top} J_g (y^k) + \lambda I$. Then {\ournet}'s
update reads $y^{k + 1} = y^k - H_k^{- 1} \nabla F (y^k)$ and \eqref{eq:pf-nonlinear-1} becomes
\begin{align}
  F (y^{k + 1}) & \leq F (y^k) + \langle \nabla F (y^k), y^{k + 1} - y^k
  \rangle + \tfrac{1}{2} (y^{k + 1} - y^k)^{\top} J_g (y^k)^{\top} J_g (y^k)
  (y^{k + 1} - y^k)\nonumber\\
  & \quad + \left( L_g M + \tfrac{L^2_g M^2}{32 \lambda} \right) \| y^{k + 1} -
  y^k \|^2 \tag{by \eqref{eq:pf-nonlinear-2} and \eqref{eq:pf-nonlinear-3}}\\
  & \leq F (y^k) + \langle \nabla F (y^k), y^{k + 1} - y^k \rangle +
  \tfrac{1}{2} (y^{k + 1} - y^k)^{\top} H_k (y^{k + 1} - y^k) \tag{by $\lambda^2 - L_g M \lambda - \tfrac{L^2_g M^2}{32} \geq 0$}\\
  & = F (y^k) - \tfrac{1}{2} \| \nabla F (y^k) \|_{H^{- 1}_k}^2, \label{eq:pf-nonlinear-4}
\end{align}
which guarantees $F (y^{k + 1}) \leq F (y^k)$ and thus $y^{k + 1} \in
\mathcal{D}$. 
The claim that $y^k \in \mathcal{D}$ for all $k \geq 0$ is proved by induction and the descent property is established.

Next, we prove the sublinear convergence using the convexity of $F$ and the descent property \eqref{eq:pf-nonlinear-4}.
Since $F$ is convex and the feasible set $\mathcal{C}$ is non-empty, we have
$\min_{y \in \mathbb{R}^n} F (y) = 0$ and
\begin{equation} \label{eq:pf-nonlinear-5}
    F (y^k) = F (y^k) - \min_{y \in \mathbb{R}^n} F (y) \leq \| \nabla F (y^k)
   \|_{H_k^{- 1}} \| y^k -\mathcal{P}_{\mathcal{C}} (y^k) \|_{H_k}.
\end{equation}
To bound $\| y^k -\mathcal{P}_{\mathcal{C}} (y^k) \|_{H_k}$, observe that $\| H_k \| \leq \sigma^2_{\max} (J_g (y^k)) + \lambda \leq \sigma^2_{\max}(J_h (y^k)) + \sigma^2_{\max} (A) + \lambda \leq \sigma^2 + \lambda$ since $y^k \in \mcal{D}$, and thus
\begin{equation} \label{eq:pf-nonlinear-6}
    \| y^k -\mathcal{P}_{\mathcal{C}} (y^k)
    \|^2_{H_k} \leq \| H_k \| \| y^k -\mathcal{P}_{\mathcal{C}} (y^k) \|^2 
    \leq (\sigma^2 + \lambda) D^2.
\end{equation}
Combine \eqref{eq:pf-nonlinear-4}, \eqref{eq:pf-nonlinear-5}, and \eqref{eq:pf-nonlinear-6} to obtain
\begin{equation} \label{eq:pf-nonlinear-7}
    F (y^{k + 1}) \leq F (y^k) - \frac{F (y^k)^2}{2 D^2 (\sigma^2 + \lambda)} .
\end{equation}
If $F (y^{k + 1}) = 0$, the desired bound follows immediately. Otherwise, \eqref{eq:pf-nonlinear-7} implies
\begin{equation*} 
    \frac{1}{F (y^{k + 1})} - \frac{1}{F (y^k)} \geq \frac{F (y^k) - F (y^{k +
   1})}{F (y^k) F (y^{k + 1})} \geq \frac{F (y^k)}{2 D^2 (\sigma^2 + \lambda)
   F (y^{k + 1})} \geq \frac{1}{2 D^2 (\sigma^2 + \lambda)} .
\end{equation*}
The telescoping sum yields
\begin{equation*} 
    \frac{1}{F (y^{k + 1})} \geq \frac{1}{F (y^0)} + \frac{k + 1}{2 D^2
   (\sigma^2 + \lambda)},
\end{equation*}
which implies the desired sublinear convergence bound.
\end{proof}
\section{Details of Baselines} \label{app:baselines}

This section provides the details of each baselines we compared to in \Cref{sec:exp}.

\begin{itemize}    
    \item {\dc} \cite{donti2021dc3} strictly enforces equality constraints by predicting free variables and solving dependent variables from equality constraints.
    The inequality constraint violation are iteratively reduced by a fixed number of gradient descent steps on the free variables.
    
    \item {\hardnet} \cite{min2024hardnet} uses soft-epochs and strictly enforces linear constraints by \eqref{eq:hardnet-layer}.
    We compare to {\hardnet} on NCLPs.

    \item {\optnet} \cite{amos2017optnet} appends a differentiable convex projection problem \eqref{eq:opt-proj} as a repair layer. We compare to {\optnet} on convex constrained problems: NCLPs and convex QCQPs.

    \item {\hproj} \cite{liang2024homeomorphic} is a post-processing method with two stages training: the first stage trains a homeomorphic map and the second stage trains a soft-constraint NN. At inference 
    time, {\hproj} uses the learned homeomorphic map and bisection to project onto the feasible set.
\end{itemize}
\section{Tables for All Test Results in \Cref{sec:opt-learning}} \label{app:tables}

This section includes all the evaluation results on test instances for experiments in \Cref{sec:opt-learning}. For ease of navigation, \Cref{tab:correspondence} summarizes the correspondence between figures presented in the main text and their associated results tables organized by problem class. Columns ``\# of Ineq. Violations'' and ``\# of Eq. Violations'' are counted with threshold $10^{-4}$.

\begin{table}[htbp]
\centering
\begin{tabular}{lll}
\toprule
Tables & Problem Class & Figures \\
\midrule
\Cref{table:NCLP_test} & NCLP & \Cref{fig:test-bar} \\
\Cref{table:qcqp_test} & Convex QCQP & \Cref{fig:test-bar} \\
\Cref{table:noncvxqcqp_test} & Non-Convex QCQP & \Cref{fig:test-bar} \\
\Cref{table:NCLP_lambda} & NCLP & \Cref{fig:noncvx-lambda} \\
\Cref{table:qcqp_10} & Convex QCQP (10 inequality constraints) & \Cref{fig:cvx_qcqp_ineq} \\
\Cref{table:qcqp_50} & Convex QCQP (50 inequality constraints) & \Cref{fig:cvx_qcqp_ineq} \\
\Cref{table:qcqp_100} & Convex QCQP (100 inequality constraints) & \Cref{fig:cvx_qcqp_ineq} \\
\bottomrule
\end{tabular}
\caption{Correspondence between tables, problem classes, and figures in main paper.}
\label{tab:correspondence}
\end{table}


\begin{table*}[htbp]
\centering
\caption{Evaluation metrics on the NCLPs test set.}
\label{table:NCLP_test}
\resizebox{\linewidth}{!}{
\begingroup
\begin{tabular}{lccccccccc}
\toprule
Method & Max Opt. Gap & GMean Opt. Gap & Max Ineq. Error & GMean Ineq. Error & \# Ineq Violations & Max Eq. Error & GMean Eq. Error & \# Eq Violations & Test Time (s) \\
\midrule
DC3 & $8.81 \pm 5.68$ & $2.68 \pm 1.35$ & $(7.16 \pm 16.00) \times 10^{-3}$ & $(1.00 \pm 0.00) \times 10^{-16}$ & $(2.40 \pm 5.37) \times 10^{-4}$ & $(6.03 \pm 1.61) \times 10^{-14}$ & $(6.46 \pm 2.27) \times 10^{-15}$ & $0.00$ & $0.186 \pm 0.074$ \\
HardNet & $(1.02 \pm 0.47) \times 10^{1}$ & $(9.17 \pm 0.99) \times 10^{-1}$ & $(7.09 \pm 1.19) \times 10^{-14}$ & $(1.10 \pm 0.01) \times 10^{-16}$ & $0.00$ & $(1.13 \pm 0.14) \times 10^{-13}$ & $(1.04 \pm 0.03) \times 10^{-14}$ & $0.00$ & $0.116 \pm 0.035$ \\
HProj & $(5.97 \pm 12.04) \times 10^{2}$ & $3.72 \pm 7.12$ & $0.00$ & $0.00$ & $0.00$ & $(2.24 \pm 4.49) \times 10^{-12}$ & $(3.12 \pm 5.58) \times 10^{-14}$ & $0.00$ & $0.049 \pm 0.029$ \\
OptNet & $(6.83 \pm 0.70) \times 10^{-1}$ & $(2.40 \pm 0.03) \times 10^{-1}$ & $0.00$ & $0.00$ & $0.00$ & $(2.24 \pm 0.57) \times 10^{-11}$ & $(3.39 \pm 0.16) \times 10^{-13}$ & $0.00$ & $1.634 \pm 0.249$ \\
SnareNet ($\mathtt{tol}=10^{-4}$) & $(4.57 \pm 0.62) \times 10^{-1}$ & $(8.03 \pm 0.58) \times 10^{-2}$ & $(5.29 \pm 0.67) \times 10^{-5}$ & $(1.02 \pm 0.40) \times 10^{-15}$ & $0.00$ & $(8.34 \pm 1.19) \times 10^{-5}$ & $(2.08 \pm 1.01) \times 10^{-6}$ & $0.00$ & $0.501 \pm 0.076$ \\
SnareNet ($\mathtt{tol}=10^{-6}$) & $(4.58 \pm 0.64) \times 10^{-1}$ & $(8.03 \pm 0.58) \times 10^{-2}$ & $(5.54 \pm 0.63) \times 10^{-7}$ & $(6.34 \pm 2.06) \times 10^{-16}$ & $0.00$ & $(8.72 \pm 0.99) \times 10^{-7}$ & $(1.49 \pm 0.78) \times 10^{-8}$ & $0.00$ & $0.869 \pm 0.112$ \\
SnareNet ($\mathtt{tol}=10^{-8}$) & $(4.57 \pm 0.64) \times 10^{-1}$ & $(8.03 \pm 0.58) \times 10^{-2}$ & $(5.36 \pm 1.09) \times 10^{-9}$ & $(3.91 \pm 0.94) \times 10^{-16}$ & $0.00$ & $(8.37 \pm 1.48) \times 10^{-9}$ & $(9.11 \pm 2.89) \times 10^{-11}$ & $0.00$ & $0.972 \pm 0.152$ \\
SnareNet ($\mathtt{tol}=10^{-10}$) & $(4.57 \pm 0.64) \times 10^{-1}$ & $(8.03 \pm 0.58) \times 10^{-2}$ & $(5.25 \pm 1.43) \times 10^{-11}$ & $(2.44 \pm 0.43) \times 10^{-16}$ & $0.00$ & $(8.00 \pm 1.66) \times 10^{-11}$ & $(6.02 \pm 2.51) \times 10^{-13}$ & $0.00$ & $1.379 \pm 0.272$ \\
SnareNet ($\mathtt{tol}=10^{-12}$) & $(4.56 \pm 0.64) \times 10^{-1}$ & $(8.03 \pm 0.58) \times 10^{-2}$ & $(5.05 \pm 1.16) \times 10^{-13}$ & $(1.51 \pm 0.17) \times 10^{-16}$ & $0.00$ & $(7.72 \pm 1.52) \times 10^{-13}$ & $(6.42 \pm 2.35) \times 10^{-15}$ & $0.00$ & $1.646 \pm 0.605$ \\
\bottomrule
\end{tabular}
\endgroup
}
\end{table*}

\begin{table*}[htbp]
\centering
\caption{Evaluation metrics on the convex QCQPs test set.}
\label{table:qcqp_test}
\resizebox{\linewidth}{!}{
\begingroup
\begin{tabular}{lccccccccc}
\toprule
Method & Max Opt. Gap & GMean Opt. Gap & Max Ineq. Error & GMean Ineq. Error & \# Ineq Violations & Max Eq. Error & GMean Eq. Error & \# Eq Violations & Test Time (s) \\
\midrule
DC3 & $(1.50 \pm 1.10) \times 10^{1}$ & $9.60 \pm 8.14$ & $2.89 \pm 4.08$ & $(2.22 \pm 2.16) \times 10^{-16}$ & $0.70 \pm 1.08$ & $(4.62 \pm 4.23) \times 10^{-14}$ & $(4.12 \pm 2.79) \times 10^{-15}$ & $0.00$ & $0.191 \pm 0.088$ \\
HProj & $(7.78 \pm 17.18) \times 10^{1}$ & $(3.41 \pm 7.57) \times 10^{1}$ & $(3.58 \pm 8.01) \times 10^{2}$ & $(1.10 \pm 2.47) \times 10^{-13}$ & $10.00 \pm 22.36$ & $(2.01 \pm 4.41) \times 10^{-12}$ & $(2.77 \pm 6.08) \times 10^{-13}$ & $0.00$ & $0.260 \pm 0.374$ \\
OptNet & $(1.21 \pm 0.10) \times 10^{-1}$ & $(4.15 \pm 0.01) \times 10^{-2}$ & $(8.09 \pm 2.07) \times 10^{-8}$ & $(2.10 \pm 0.03) \times 10^{-14}$ & $0.00$ & $(1.25 \pm 0.13) \times 10^{-11}$ & $(2.51 \pm 0.16) \times 10^{-13}$ & $0.00$ & $3.044 \pm 0.203$ \\
SnareNet ($\mathtt{tol}=10^{-4}$) & $(1.24 \pm 0.06) \times 10^{-1}$ & $(2.60 \pm 0.26) \times 10^{-2}$ & $(8.63 \pm 0.94) \times 10^{-5}$ & $(7.74 \pm 1.97) \times 10^{-15}$ & $0.00$ & $(8.33 \pm 0.87) \times 10^{-5}$ & $(1.93 \pm 0.45) \times 10^{-6}$ & $0.00$ & $0.807 \pm 0.093$ \\
SnareNet ($\mathtt{tol}=10^{-6}$) & $(1.25 \pm 0.06) \times 10^{-1}$ & $(2.60 \pm 0.27) \times 10^{-2}$ & $(7.91 \pm 0.98) \times 10^{-7}$ & $(2.68 \pm 0.64) \times 10^{-15}$ & $0.00$ & $(7.64 \pm 0.97) \times 10^{-7}$ & $(5.92 \pm 2.64) \times 10^{-9}$ & $0.00$ & $1.356 \pm 0.135$ \\
SnareNet ($\mathtt{tol}=10^{-8}$) & $(1.24 \pm 0.07) \times 10^{-1}$ & $(2.60 \pm 0.26) \times 10^{-2}$ & $(8.15 \pm 0.85) \times 10^{-9}$ & $(9.52 \pm 1.98) \times 10^{-16}$ & $0.00$ & $(7.88 \pm 0.82) \times 10^{-9}$ & $(2.13 \pm 1.23) \times 10^{-11}$ & $0.00$ & $1.780 \pm 0.183$ \\
SnareNet ($\mathtt{tol}=10^{-10}$) & $(1.25 \pm 0.07) \times 10^{-1}$ & $(2.60 \pm 0.26) \times 10^{-2}$ & $(7.56 \pm 0.70) \times 10^{-11}$ & $(3.35 \pm 0.70) \times 10^{-16}$ & $0.00$ & $(7.31 \pm 0.72) \times 10^{-11}$ & $(7.35 \pm 5.48) \times 10^{-14}$ & $0.00$ & $2.231 \pm 0.454$ \\
SnareNet ($\mathtt{tol}=10^{-12}$) & $(1.24 \pm 0.06) \times 10^{-1}$ & $(2.60 \pm 0.26) \times 10^{-2}$ & $(7.38 \pm 0.99) \times 10^{-13}$ & $(1.36 \pm 0.17) \times 10^{-16}$ & $0.00$ & $(7.06 \pm 0.95) \times 10^{-13}$ & $(9.93 \pm 3.87) \times 10^{-16}$ & $0.00$ & $2.603 \pm 0.956$ \\
\bottomrule
\end{tabular}
\endgroup
}
\end{table*}

\begin{table*}[htbp]
\centering
\caption{Evaluation metrics on the non-convex QCQPs test set.}
\label{table:noncvxqcqp_test}
\resizebox{\linewidth}{!}{
\begingroup
\begin{tabular}{lccccccccc}
\toprule
Method & Max Obj. Value & Mean Obj. Value & Max Ineq. Error & GMean Ineq. Error & \# Ineq Violations & Max Eq. Error & GMean Eq. Error & \# Eq Violations & Test Time (s) \\
\midrule
DC3 & $-12.70 \pm 9.39$ & $-15.25 \pm 8.99$ & $(6.66 \pm 10.43) \times 10^{-1}$ & $(1.09 \pm 0.09) \times 10^{-16}$ & $0.12 \pm 0.12$ & $(7.71 \pm 2.16) \times 10^{-15}$ & $(7.28 \pm 0.36) \times 10^{-16}$ & $0.00$ & $0.172 \pm 0.059$ \\
HProj & $12.86 \pm 8.93$ & $-0.47 \pm 9.22$ & $(1.69 \pm 0.70) \times 10^{1}$ & $(3.55 \pm 6.15) \times 10^{-11}$ & $5.85 \pm 8.50$ & $(1.05 \pm 0.70) \times 10^{-14}$ & $(1.05 \pm 0.60) \times 10^{-15}$ & $0.00$ & $0.438 \pm 0.139$ \\
SnareNet ($\mathtt{tol}=10^{-4}$) & $-19.98 \pm 1.89$ & $-20.61 \pm 1.81$ & $(2.40 \pm 1.89) \times 10^{-6}$ & $(6.19 \pm 10.78) \times 10^{-10}$ & $0.00$ & $(1.07 \pm 0.84) \times 10^{-6}$ & $(3.87 \pm 2.92) \times 10^{-7}$ & $0.00$ & $0.124 \pm 0.059$ \\
SnareNet ($\mathtt{tol}=10^{-6}$) & $-21.16 \pm 1.37$ & $-21.72 \pm 1.32$ & $(1.10 \pm 2.27) \times 10^{-7}$ & $(2.81 \pm 5.40) \times 10^{-12}$ & $0.00$ & $(4.48 \pm 9.07) \times 10^{-8}$ & $(2.14 \pm 4.40) \times 10^{-8}$ & $0.00$ & $0.124 \pm 0.020$ \\
SnareNet ($\mathtt{tol}=10^{-8}$) & $-20.61 \pm 1.98$ & $-21.18 \pm 1.91$ & $(3.68 \pm 4.05) \times 10^{-9}$ & $(2.25 \pm 1.54) \times 10^{-13}$ & $0.00$ & $(1.99 \pm 2.19) \times 10^{-9}$ & $(7.84 \pm 8.85) \times 10^{-10}$ & $0.00$ & $0.141 \pm 0.028$ \\
SnareNet ($\mathtt{tol}=10^{-10}$) & $-13.66 \pm 16.95$ & $-15.89 \pm 13.23$ & $(1.11 \pm 2.49) \times 10^{-5}$ & $(1.05 \pm 0.95) \times 10^{-14}$ & $0.00$ & $(6.28 \pm 14.03) \times 10^{-6}$ & $(2.55 \pm 2.97) \times 10^{-12}$ & $0.00$ & $0.616 \pm 1.049$ \\
SnareNet ($\mathtt{tol}=10^{-12}$) & $-13.01 \pm 18.41$ & $-17.40 \pm 9.87$ & $(1.47 \pm 3.28) \times 10^{-6}$ & $(6.86 \pm 3.82) \times 10^{-16}$ & $0.00$ & $(7.90 \pm 17.66) \times 10^{-7}$ & $(8.91 \pm 8.89) \times 10^{-15}$ & $0.00$ & $0.651 \pm 1.117$ \\
\bottomrule
\end{tabular}
\endgroup
}
\end{table*}


\begin{table*}[htbp]
\centering
\caption{Evaluation metrics on the NCLP test set.}
\label{table:NCLP_lambda}
\resizebox{\linewidth}{!}{
\begingroup
\begin{tabular}{lcccccccccc}
\toprule
Method & Obj. Value & Max Opt. Gap & GMean Opt. Gap & Max Ineq. Error & GMean Ineq. Error & \# Ineq Violations & Max Eq. Error & GMean Eq. Error & \# Eq Violations & Test Time (s) \\
\midrule
$\mathtt{SnareNet}$ ($\lambda=0.001$) & $-11.19 \pm 0.29$ & $2.92 \pm 1.73$ & $(4.28 \pm 2.66) \times 10^{-1}$ & $(4.07 \pm 1.46) \times 10^{-9}$ & $(1.72 \pm 0.14) \times 10^{-16}$ & $0.00$ & $(6.57 \pm 2.36) \times 10^{-9}$ & $(1.87 \pm 0.97) \times 10^{-10}$ & $0.00$ & $0.28 \pm 0.12$ \\
$\mathtt{SnareNet}$ ($\lambda=0.01$) & $-11.59 \pm 0.01$ & $(4.28 \pm 0.46) \times 10^{-1}$ & $(7.90 \pm 1.00) \times 10^{-2}$ & $(5.35 \pm 0.88) \times 10^{-9}$ & $(2.71 \pm 0.51) \times 10^{-16}$ & $0.00$ & $(7.88 \pm 1.14) \times 10^{-9}$ & $(7.74 \pm 3.91) \times 10^{-11}$ & $0.00$ & $0.76 \pm 0.74$ \\
$\mathtt{SnareNet}$ ($\lambda=0.1$) & $-11.64 \pm 0.00$ & $(2.18 \pm 0.12) \times 10^{-1}$ & $(3.86 \pm 0.22) \times 10^{-2}$ & $(1.40 \pm 0.22) \times 10^{-8}$ & $(3.58 \pm 0.04) \times 10^{-16}$ & $0.00$ & $(2.08 \pm 0.33) \times 10^{-8}$ & $(4.29 \pm 0.82) \times 10^{-10}$ & $0.00$ & $0.61 \pm 0.26$ \\
$\mathtt{SnareNet}$ ($\lambda=10$) & $-11.67 \pm 0.00$ & $(1.87 \pm 0.16) \times 10^{-2}$ & $(4.29 \pm 0.06) \times 10^{-3}$ & $(8.65 \pm 0.46) \times 10^{-8}$ & $(6.38 \pm 0.16) \times 10^{-16}$ & $0.00$ & $(1.20 \pm 0.06) \times 10^{-7}$ & $(4.98 \pm 0.61) \times 10^{-9}$ & $0.00$ & $2.19 \pm 1.24$ \\
$\mathtt{SnareNet}$ ($\lambda=1$) & $-11.66 \pm 0.00$ & $(6.99 \pm 0.74) \times 10^{-2}$ & $(1.57 \pm 0.05) \times 10^{-2}$ & $(1.96 \pm 0.08) \times 10^{-8}$ & $(5.01 \pm 0.09) \times 10^{-16}$ & $0.00$ & $(3.21 \pm 0.13) \times 10^{-8}$ & $(6.27 \pm 0.93) \times 10^{-10}$ & $0.00$ & $1.31 \pm 0.32$ \\
$\mathtt{SnareNet}$ ($\lambda=5$) & $-11.67 \pm 0.00$ & $(2.74 \pm 0.08) \times 10^{-2}$ & $(6.34 \pm 0.06) \times 10^{-3}$ & $(5.07 \pm 0.25) \times 10^{-8}$ & $(6.78 \pm 0.12) \times 10^{-16}$ & $0.00$ & $(7.52 \pm 0.39) \times 10^{-8}$ & $(3.19 \pm 0.24) \times 10^{-9}$ & $0.00$ & $2.22 \pm 1.17$ \\
\bottomrule
\end{tabular}
\endgroup
}
\end{table*}


\begin{table*}[htbp]
\centering
\caption{Evaluation metrics on convex QCQP with 10 inequality constraints test set.}
\label{table:qcqp_10}
\resizebox{\linewidth}{!}{
\begingroup
\begin{tabular}{lcccccccc}
\toprule
Method & Max Opt. Gap & GMean Opt. Gap & GMean Ineq. Error & \# Ineq Violations & Max Eq. Error & GMean Eq. Error & \# Eq Violations & Test Time (s) \\
\midrule
Optimizer (Gurobi) & $1.42 \times 10^{-14}$ & $4.48 \times 10^{-16}$ & $0.00$ & $0.00$ & $5.43 \times 10^{-10}$ & $1.06 \times 10^{-14}$ & $0.00$ & $428.44$ \\
DC3 & $(7.96 \pm 11.33) \times 10^{1}$ & $(2.21 \pm 1.94) \times 10^{1}$ & $(3.52 \pm 7.87) \times 10^{-10}$ & $0.97 \pm 1.88$ & $(1.48 \pm 1.11) \times 10^{-13}$ & $(1.40 \pm 0.94) \times 10^{-14}$ & $0.00$ & $0.08 \pm 0.07$ \\
HProj & $(2.40 \pm 1.99) \times 10^{-1}$ & $(1.60 \pm 1.42) \times 10^{-1}$ & $0.00$ & $0.00$ & $(1.08 \pm 0.92) \times 10^{-13}$ & $(1.41 \pm 0.98) \times 10^{-14}$ & $0.00$ & $0.02 \pm 0.02$ \\
SnareNet (Ours) & $(1.74 \pm 0.08) \times 10^{-1}$ & $(3.16 \pm 0.29) \times 10^{-2}$ & $(4.78 \pm 0.92) \times 10^{-15}$ & $0.00$ & $(5.13 \pm 1.28) \times 10^{-5}$ & $(6.35 \pm 1.68) \times 10^{-7}$ & $0.00$ & $0.25 \pm 0.01$ \\
\bottomrule
\end{tabular}
\endgroup
}
\end{table*}

\begin{table*}[htbp]
\centering
\caption{Evaluation metrics on convex QCQP with 50 inequality constraints test set.}
\label{table:qcqp_50}
\resizebox{\linewidth}{!}{
\begingroup
\begin{tabular}{lcccccccc}
\toprule
Method & Max Opt. Gap & GMean Opt. Gap & GMean Ineq. Error & \# Ineq Violations & Max Eq. Error & GMean Eq. Error & \# Eq Violations & Test Time (s) \\
\midrule
Optimizer (Gurobi) & $7.11 \times 10^{-15}$ & $2.74 \times 10^{-16}$ & $1.01 \times 10^{-16}$ & $0.00$ & $1.58 \times 10^{-8}$ & $3.89 \times 10^{-13}$ & $0.00$ & $1742.79$ \\
DC3 & $(3.37 \pm 1.78) \times 10^{1}$ & $(1.59 \pm 0.76) \times 10^{1}$ & $(2.39 \pm 4.72) \times 10^{-13}$ & $5.33 \pm 5.07$ & $(7.63 \pm 5.12) \times 10^{-14}$ & $(5.42 \pm 3.01) \times 10^{-15}$ & $0.00$ & $0.18 \pm 0.08$ \\
HProj & $(7.78 \pm 17.18) \times 10^{1}$ & $(3.41 \pm 7.57) \times 10^{1}$ & $(1.10 \pm 2.47) \times 10^{-13}$ & $10.00 \pm 22.36$ & $(2.01 \pm 4.41) \times 10^{-12}$ & $(2.77 \pm 6.08) \times 10^{-13}$ & $0.00$ & $0.26 \pm 0.37$ \\
SnareNet (Ours) & $(1.25 \pm 0.06) \times 10^{-1}$ & $(2.60 \pm 0.27) \times 10^{-2}$ & $(2.68 \pm 0.64) \times 10^{-15}$ & $0.00$ & $(7.64 \pm 0.97) \times 10^{-7}$ & $(5.92 \pm 2.64) \times 10^{-9}$ & $0.00$ & $0.39 \pm 0.01$ \\
\bottomrule
\end{tabular}
\endgroup
}
\end{table*}

\begin{table*}[htbp]
\centering
\caption{Evaluation metrics on convex QCQP with 100 inequality constraints test set.}
\label{table:qcqp_100}
\resizebox{\linewidth}{!}{
\begingroup
\begin{tabular}{lcccccccc}
\toprule
Method & Max Opt. Gap & GMean Opt. Gap & GMean Ineq. Error & \# Ineq Violations & Max Eq. Error & GMean Eq. Error & \# Eq Violations & Test Time (s) \\
\midrule
Optimizer (Gurobi) & $7.11 \times 10^{-15}$ & $2.57 \times 10^{-16}$ & $1.04 \times 10^{-16}$ & $0.00$ & $1.57 \times 10^{-7}$ & $7.90 \times 10^{-13}$ & $0.00$ & $3563.30$ \\
DC3 & $(1.14 \pm 0.66) \times 10^{1}$ & $6.86 \pm 4.63$ & $(1.42 \pm 0.91) \times 10^{-16}$ & $0.65 \pm 1.37$ & $(1.76 \pm 1.27) \times 10^{-14}$ & $(1.70 \pm 0.89) \times 10^{-15}$ & $0.00$ & $0.48 \pm 0.08$ \\
HProj & $(2.35 \pm 5.22) \times 10^{2}$ & $(1.53 \pm 3.40) \times 10^{2}$ & $(5.40 \pm 12.08) \times 10^{-11}$ & $20.07 \pm 44.86$ & $(1.83 \pm 4.03) \times 10^{-12}$ & $(3.04 \pm 6.73) \times 10^{-13}$ & $0.00$ & $0.44 \pm 0.69$ \\
SnareNet (Ours) & $(1.74 \pm 0.08) \times 10^{-1}$ & $(3.16 \pm 0.29) \times 10^{-2}$ & $(1.88 \pm 0.22) \times 10^{-15}$ & $0.00$ & $(7.77 \pm 0.54) \times 10^{-7}$ & $(2.80 \pm 0.97) \times 10^{-9}$ & $0.00$ & $0.68 \pm 0.12$ \\
\bottomrule
\end{tabular}
\endgroup
}
\end{table*}
\section{Neural Control Policies Experiment Details} \label{app:cbf}

The system state is $x(t) = ( p_x(t), p_y(t), \theta(t), v(t), w(t)) \in \mbb{R}^5$ at time $t$, where $p_x, p_y$ are the coordinates, $\theta$ is the heading angle, $v$ is the linear velocity, and $w$ is the angular velocity.
The control $u(t) = (a(t), \alpha(t)) \in \mbb{R}^2$ includes linear acceleration $a$ and angular acceleration $\alpha$.
The control changes the system state by the linear dynamics:
$\dot{x} = F(x) + G u$,
where $F: \mbb{R}^5 \rightarrow \mbb{R}^5$ and $G \in \mbb{R}^{5 \times 2}$ are 
\begin{align*}
    F(x) = \begin{bmatrix}
        v \cos \theta\\
        v \sin \theta\\
        w\\
        0\\
        0
    \end{bmatrix} \quad \text{ and } \quad 
    G(x) = \begin{bmatrix}
        0 & 0\\
        0 & 0\\
        0 & 0\\
        1 & 0\\
        0 & 1\\
    \end{bmatrix}.
\end{align*}
Our neural policy is trained by minimizing a quadratic cost over a finite time horizon:
\begin{equation*}
    \Delta t\sum_{i=0}^{n_t} x(t_i)^T Q x(t_i) + c \|\pi_{\theta}(x(t_i))\|^2,
\end{equation*}
where $Q = \text{diag}(100,100,0,0.1,0.1)$ penalizes deviations from the target position $(0,0)$ and velocity, $c = 0.1$ weights the control effort, and $n_t = 10$ steps for trajectories.
This objective encourages trajectories to reach the origin $(0,0)$ with minimal control effort.

For an elliptical obstacle centered at $(c_x,c_y)$ with axes $(r_x,r_y)$, we use the higher order CBF $h(x) = \dot{h}_e(x) + \kappa h_e(x)$ with $\kappa > 0$, \cite{min2024hardnet} in our experiments, where
\begin{align*}
    h_e(x) = \big(\frac{c_x-p_x+\ell \cos \theta}{r_x}\big)^2+\big(\frac{c_y-p_y+\ell \sin \theta}{r_y}\big)^2-1.
\end{align*}
For training, we uniformly sample initial states from two boxes in the state space, differing only in position: Box 1 with positions $(p_x, p_y) \in [-5.0, -5.5] \times [7.5, 8.0]$ and Box 2 with positions $(p_x, p_y) \in [-8.5, -8.0] \times [5.5, 6.0]$.
Both boxes share the same heading angle range $\theta \in [-\frac{\pi}{4}, -\frac{\pi}{8}]$ and zero initial velocities ($v = w = 0$).

\Cref{tab:cbf_base_seed89_dc3_snarenet} shows the detailed evaluation statistics: {\ournet} achieves a 99.4\% feasibility rate on the test set compared to {\dc}'s 64.9\% feasibility rate, demonstrating the effectiveness of {\ournet} in enforcing safety constraints in neural control policies.

\begin{table}[H]
\centering
\caption{Evaluation statistics on 84 CBF test instances.}
\label{tab:cbf_base_seed89_dc3_snarenet}
\begin{tabular}{lccccc}
\toprule
Method & Dataset & Objective & \# Feasible (1e-4) & Max Ineq. Error & GMean Ineq. Error \\
\midrule
DC3 & Box 1 & $3656$ & $58$ & $6.555$ & $6.785 \times 10^{-15}$ \\
SnareNet & Box 1 & $4060$ & $84$ & $4.21 \times 10^{-8}$ & $1.763 \times 10^{-14}$ \\
\midrule
DC3 & Box 2 & $4269$ & $51$ & $4.17$ & $1.913 \times 10^{-14}$ \\
SnareNet & Box 2 & $5226$ & $83$ & $0.7483$ & $4.038 \times 10^{-14}$ \\
\bottomrule
\end{tabular}
\end{table}
\section{Algorithms of {\ournet}}

\Cref{alg:repair} provides the pseudocode for {\ournet}'s repair layer. 
The training algorithm of {\ournet} is summarized in \Cref{alg:ournet}.

\begin{algorithm}[t]
\caption{Repair layer in {\ournet}}
\label{alg:repair}
\begin{algorithmic}[1]
\STATE \textbf{assume} $\hat{y} = \mcal{M}_{\theta}(x)$ and constraints $\ell \leq g(y) \leq u$ 

\FUNCTION{$\mcal{R}$($\hat{y}, \lambda, \varepsilon$)}
    \STATE \textbf{init} $y^0 = \hat{y}$
    \FOR{$k = 0, 1, 2, \ldots$ until convergence}
        \STATE \textbf{compute} $z^k$ using \cref{eq:z-adarelx} with $\varepsilon$
        \STATE \textbf{update} $y^{k+1}$ using \cref{eq:LM-update} with $\lambda$
    \ENDFOR
    \STATE \textbf{return} $\check{y} = y^{k+1}$
\ENDFUNCTION
\end{algorithmic}
\end{algorithm}

\begin{algorithm}[t]
\caption{Training paradigm of {\ournet}}
\label{alg:ournet}
\begin{algorithmic}[1]
\FUNCTION{TRAIN($\mcal{X}_{\train}$)}
    \STATE \textbf{init} neural network $\mcal{M}_\theta : \mathbb{R}^d \to \mathbb{R}^n$ and $\lambda \geq 0$
    \STATE \textbf{init} constraint violation $\varepsilon_x^{(0)}$ for $x \in \mcal{X}_{\train}$
    \STATE \textbf{set} a relaxation parameter schedule $\{\varepsilon_x^{(t)}\}_{t=1}^{T}$
    \FOR{epoch $t = 1, 2, \ldots, T$}
        \FOR{mini-batch $\mcal{B} \subset \mcal{X}_{\train}$}
            \STATE \textbf{compute} approximate solution $\hat{y}_i = \mcal{M}_\theta(x_i)$ for all $x_i \in \mbf{B}$
            \STATE \textbf{repair} to relaxed feasible solution $\check{y}_i = \mcal{R}(\hat{y}_i, \lambda, \varepsilon_{x_i}^{(t)})$ for all $x_i \in \mbf{B}$
            \STATE \textbf{compute} batch loss $\mcal{L}_{\mcal{B}}(\theta)$
            \STATE \textbf{update} $\theta$ using $\nabla_\theta \mcal{L}_{\mcal{B}}(\theta)$
        \ENDFOR
    \ENDFOR
\ENDFUNCTION
\end{algorithmic}
\end{algorithm}
\section{Soft Constraint Training Can Be Counterproductive} \label{app:soft}

Soft constraint training is commonly employed as a warm-up strategy for hard-constrained neural networks. 
However, as shown in \Cref{fig:soft_hardnet_max}, this approach can be counterproductive. 
While {\hardnet} achieves a lower optimality gap after 1000 soft epochs than after 500, 
this reduction does not persist: 
once hard constraint training begins, the optimality gap increases sharply, negating any apparent progress. 
Additional soft constraint epochs constitute an inefficient use of computational resources, as the gains from soft constraint training do not transfer to improved final model performance.

\begin{figure*}[h]
    \centering
    \includegraphics[width=0.9\textwidth]{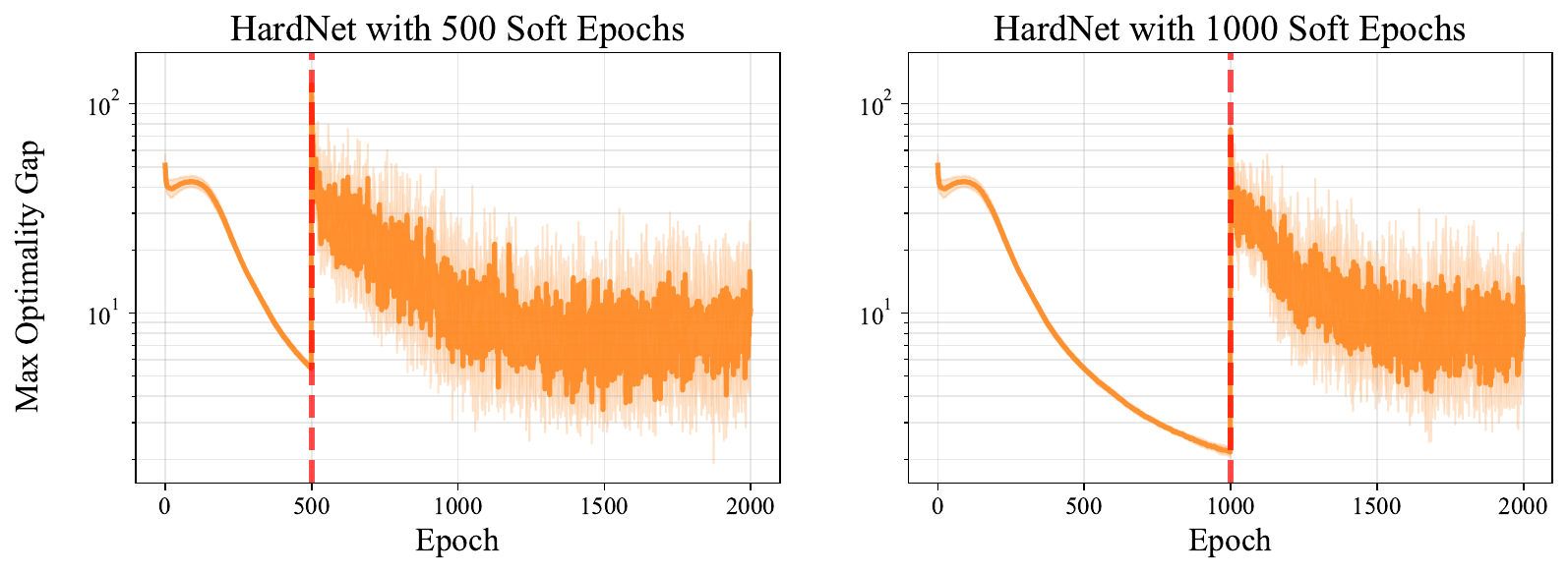}
    \caption{Maximum optimality gap over all NCP validation instances for {\hardnet} trained with soft constraints for 500 and 1000 epochs and hard constraints for the rest.}
    \label{fig:soft_hardnet_max}
\end{figure*}

\section{Scaling of Computational Resources} \label{app:scaling}

\Cref{tab:noncvx_scaling_metrics} shows the training time, memory usage, and number of repair iterations taken for SnareNet on non-convex programs (NCPs) of varying sizes. 
The experiments were run on an Intel 6426Y CPU and an NVIDIA L40S GPU, constrained to 16, 16, and 32 GB of GPU memory per run for 100, 500, and 1000 constraints, respectively. 
Note that SnareNet may use less computational resources for larger problems under a fixed number of constraints since the problems can be less constrained and require fewer repair iterations to achieve feasibility.
For example, the 500 and 1000 variable problems with 100 constraints are cheaper (both in time and memory) than the 100 variable problem with 100 constraints.
\Cref{tab:noncvx_scaling_metrics} indicates that SnareNet is more suitable for problems with $m \ll n$.

\begin{table}[H]
\centering
\caption{Total training time (sec) / memory usage / maximum repair iterations on NCPs of varying sizes.}
\label{tab:noncvx_scaling_metrics}
\begin{tabular}{llll}
\toprule
 & \shortstack{100 constraints \\ (50 eq. + 50 ineq.)} & \shortstack{500 constraints \\ (50 eq. + 450 ineq.)} & \shortstack{1000 constraints \\ (50 eq. + 950 ineq.)} \\
\midrule
100 var & 3674.4s / 2.3GB / 32 & 3496.9s / 4.4GB / 100 & 5832.5s / 4.9GB / 100 \\
500 var & 869.1s / 1.3GB / 3 & 6493.8s / 11.6GB / 11 & 29301.7s / 23.8GB / 30 \\
1000 var & 1107.3s / 1.3GB / 3 & 2724.5s / 3.3GB / 3 & $>$86400.0s / 31.1GB / 8 \\
\bottomrule
\end{tabular}
\end{table}


\end{document}